\documentclass[sigconf, nonacm, nonatbib]{acmart}

\usepackage{microtype}
\usepackage{graphicx}
\usepackage{subfigure}
\usepackage{booktabs} 

\usepackage[utf8]{inputenc} 
\usepackage[T1]{fontenc}    
\usepackage{hyperref}       
\usepackage{url}            
\usepackage{amsfonts}       
\usepackage{nicefrac}       
\usepackage{wrapfig}
\usepackage{multirow}
\usepackage{url} 
\usepackage{xcolor}
\usepackage{bm}
\usepackage{amsmath}
\usepackage{enumitem}
\usepackage{algorithm}
\usepackage{algorithmic}
\usepackage{multicol}
\usepackage{balance}

\newtheorem{assumption}{Assumption}
\newtheorem{property}{Property}
\newtheorem{lemma}{Lemma}
\newtheorem{theorem}{Theorem}
\newtheorem{corollary}{Corollary}

\newcommand\vldbdoi{10.14778/3529337.3529343}
\newcommand\vldbpages{}
\newcommand\vldbvolume{15}
\newcommand\vldbissue{8}
\newcommand\vldbyear{2022}
\newcommand\vldbauthors{\authors}
\newcommand\vldbtitle{\shorttitle} 
\newcommand\vldbavailabilityurl{URL_TO_YOUR_ARTIFACTS}
\newcommand\vldbpagestyle{empty} 

\begin{document}
\title{Distributed Learning of Fully Connected Neural Networks using Independent Subnet Training}

\author{Binhang Yuan}
\affiliation{%
  \institution{Rice University}
}
\email{by8@rice.edu}

\author{Cameron R. Wolfe}
\affiliation{%
  \institution{Rice University}
}
\email{crw13@rice.edu}

\author{Chen Dun}
\affiliation{%
  \institution{Rice University}
}
\email{cd46@rice.edu}

\author{Yuxin Tang}
\affiliation{%
  \institution{Rice University}
}
\email{yuxin.tang@rice.edu}

\author{Anastasios Kyrillidis}
\affiliation{%
  \institution{Rice University}
}
\email{anastasios@rice.edu}

\author{Chris Jermaine}
\affiliation{%
  \institution{Rice University}
}
\email{cmj4@rice.edu}

\renewcommand{\shortauthors}{Binhang, et al.}
\begin{abstract}
Distributed machine learning (ML) can bring more computational resources to bear than single-machine learning, thus enabling reductions in training time. 
Distributed learning partitions models and data over many machines, allowing model and dataset sizes beyond the available compute power and memory of a single machine.
In practice though, distributed ML is challenging when distribution is mandatory, rather than chosen by the practitioner.
In such scenarios, data could unavoidably be separated among workers due to limited memory capacity per worker or even because of data privacy issues. 
There, existing distributed methods will utterly fail due to dominant transfer costs across workers, or do not even apply.

We propose a new approach to distributed fully connected neural network learning, called independent subnet training (IST), to handle these cases. 
In IST, the original network is decomposed into a set of narrow subnetworks with the same depth.
These subnetworks are then trained locally before parameters are exchanged to produce new subnets and the training cycle repeats.
Such a naturally ``model parallel'' approach limits memory usage by storing only a portion of network parameters on each device. 
Additionally, no requirements exist for sharing data between workers (i.e., subnet training is local and independent) and communication volume and frequency are reduced by decomposing the original network into independent subnets.
These properties of IST can cope with issues due to distributed data, slow interconnects, or limited device memory, making IST a suitable approach for cases of mandatory distribution.
We show experimentally that IST results in training times that are much lower than common distributed learning approaches. 
\end{abstract}

\maketitle

\pagestyle{\vldbpagestyle}
\begingroup\small\noindent\raggedright\textbf{PVLDB Reference Format:}\\
\vldbauthors. \vldbtitle. PVLDB, \vldbvolume(\vldbissue): \vldbpages, \vldbyear.\\
\href{https://doi.org/\vldbdoi}{doi:\vldbdoi}
\endgroup
\begingroup
\renewcommand\thefootnote{}\footnote{\noindent
This work is licensed under the Creative Commons BY-NC-ND 4.0 International License. Visit \url{https://creativecommons.org/licenses/by-nc-nd/4.0/} to view a copy of this license. For any use beyond those covered by this license, obtain permission by emailing \href{mailto:info@vldb.org}{info@vldb.org}. Copyright is held by the owner/author(s). Publication rights licensed to the VLDB Endowment. \\
\raggedright Proceedings of the VLDB Endowment, Vol. \vldbvolume, No. \vldbissue\ %
ISSN 2150-8097. \\
\href{https://doi.org/\vldbdoi}{doi:\vldbdoi} \\
}\addtocounter{footnote}{-1}\endgroup

\ifdefempty{\vldbavailabilityurl}{}{
\vspace{.3cm}
\begingroup\small\noindent\raggedright\textbf{PVLDB Artifact Availability:}\\
Code and data are available at \url{https://github.com/BinhangYuan/IST_Release}.
\endgroup
}

\section{Introduction}
Distributed training of neural networks (NN) over a compute cluster is a common task in modern computing systems \cite{ratner2019sysml,dean2012large, chilimbi2014project, li2014scaling, hadjis2016omnivore}.
Sometimes, it is the case that distributed training is a choice, and the practitioner is fully in control of the training environment.
Namely, practitioners opt for distribution with the goal of using extra hardware to lower the wall-clock time to convergence, or to allow more resources (such as memory or CPU/GPU cycles) to be brought to bear on the problem of training a model.  Consider the task of training a model such as GPT-3 \cite{brown2020language}, which requires on the order of 1000 years of GPU time to train.  Thousands of GPUs can be used to lower the time to weeks or months.  In such a training scenario, the different sites or compute units are typically connected with a high-speed network, and the hardware is often carefully tailored to the task of distributed training.

However, there are other cases where distribution is \textit{mandatory} and the hardware may be sub-optimal---very far from the idealized environment a company such as OpenAI uses to train GPT-3.
For example, consider a case where the training dataset is fragmented across several locations and organizations with privacy mandates preventing the possibility of centralized computing \cite{cordis2019machine, rieke2020future, courtiol2019deep, srinidhi2020deep, zhu2020application}.
The training set may be large, and stored across hundreds of machines \cite{chen2019federated, hard2018federated, pichai2019privacy, ramaswamy2019federated, yang2018applied, de2019federated, leroy2019federated}.

Here, the data sits where it happens to sit, and the computing environment is often not under the practitioners' control, forcing NN training to be conducted over a less-than-ideal hardware setup (e.g., too many compute nodes, CPUs, low-end GPUs, low-bandwidth interconnects, etc.).
Such scenarios arise often in practice \cite{chai2019towards, chai2020tifl, khan2021federated}. Even NN training on public compute clouds (e.g., Amazon EC2\footnote{  89\% of cloud-based deep learning projects are executed on EC2, according to Amazon's marketing materials.}) suffers from the combination of slow interconnects with high-performance GPUs \cite{mehrotra2012performance}.

We argue that common methods of distributing ML computations cannot be expected to handle such non-ideal  environments gracefully, and new methods are needed.  In distributed NN training, existing methods are roughly categorized into \emph{model parallel} and \emph{data parallel} methodologies. In practice, data parallel methodologies are most commonly used and supported due to their ease of implementation \cite{abadi2016tensorflow, paszke2017automatic}.
In model parallel training \cite{dean2012large,hadjis2016omnivore}, portions of the NN are partitioned across different compute nodes, while, in the latter \cite{zhang1990efficient, farber1997parallel, raina2009large}, the complete NN is updated with different data on each compute node.
Data parallel methods suffer when bandwidth is limited because they must transfer an entire model to each site in order to synchronize the computation.
For a large model with many parameters, this is not a reasonable requirement.
However, in typical mandatory distribution scenarios, model parallel methods are not a reasonable option, either.
When data are sharded across sites, model parallel computing implies that different parts of a model can only be updated to reflect the data present at any given node.
For these parts to stay synchronized, very fine-grained communication is required.
Thus, neither data parallel nor model parallel is fully capable of handling mandatory distribution scenarios.

\noindent
\textbf{Independent subnet training.}
In response to this, we propose independent subnet training (IST), a novel distributed training technique on fully connected NNs that combines techniques from model and data parallel training to maximize communication efficiency. 
Inspired by dropout \cite{srivastava2014dropout} and approximate matrix multiplication \cite{drineas2006fast}, IST decomposes fully-connected NN layers by distributing the neurons disjointly across different sites, forming a group of \emph{subnets}.
Then, each of these subnets is trained independently for one or more local stochastic gradient descent (SGD) iterations before synchronization \cite{lin2018don}.
After synchronization, parameters are re-distributed based on a new, random neuron sampling, and the local subnet training process repeats.

IST focuses upon the distributed training of NNs with fully-connected layers.
Such a focus has also applications in diverse NN architectures (e.g., convolutional NNs \cite{he2016deep}): the majority of NN models typically contain large, fully-connected layer, and such layers typically dominate the total number of parameters.
As such, IST can be applied to the fully-connected portion of the networks to yield a performance speedup; see Sec. \ref{S:experiments}.

In cases of mandatory distribution, model parallel training is impractical because it requires that all data is present on the server that passes data into the network's input module.\footnote{If data is fragmented across many machines, model parallel training would struggle greatly to visit the entire dataset during training, as the input module is only stored on a single node and all data used during a particular training round must be present on this node.}
Furthermore, we claim that, under mandatory distribution, IST is more capable than techniques like data parallel training due to its ability to reduce communication volume and memory usage to cope with hardware limitations.
Synchronization in IST is simply an exchange of parameters between sites\footnote{Each node in an $n$-machine cluster will receive a fraction between $\frac{1}{n^2}$ and $\frac{1}{n}$ of total model parameters under IST, while data parallel training requires \emph{all} model weights to be communicated between machines.} (i.e., no parameters are shared between subnets) and no synchronization is required during local updates, thus reducing per-step communication volume on multiple fronts.
Furthermore, IST limits its memory usage by only sending a small portion of its parameters to each device, which prevents model capacity from being limited by the  memory of a single device.

\noindent
\textbf{Contributions.}
Our proposal aggressively reduces the communication bottlenecks that plague the scalability of most popular methods of distributed NN training. 
\emph{As such, IST is most beneficial for training networks with fully-connected layers in cases of mandatory distribution, where training is highly-distributed and hardware is less-than-ideal.}
The key contributions of our work can be summarized as follows.
\begin{itemize}[leftmargin=*]
\item We propose IST, a distributed training methodology that combines ideas from model and data parallel training by breaking the original NN into a set of disjoint subnetworks that are distributed, locally trained, and re-assembled per iteration.
\item We evaluate IST on speech recognition, image classification (CIFAR10 and full ImageNet\footnote{We underline the use of the full ImageNet dataset \cite{kolesnikov2020big, codreanu2017scale} that includes 14,197,122 images, divided into 21,841 classes.}), and large-scale product recommendation tasks.
Using bandwidth-optimal ring all-reduce \cite{xu2018nccl}, IST is shown to improve time-to-convergence by as much as $10\times$ in comparison to a state-of-the-art implementation of data parallel training and ``vanilla'' local SGD \cite{lin2018don} (i.e., the only practically viable options under mandatory distribution), as well as surpass the performance of the widely-used ensemble learning method.
\item We demonstrate that IST, by enabling models with larger embedding dimensions (i.e., too large for data parallel training) to be trained, is capable of solving an ``extreme'' product recommendation task with improved generalization.
\item Finally, we theoretically show that such IST decomposition still guarantees sublinear convergence to a first-order stationary point on expectation under common assumptions.
\end{itemize}

\section{Training via Independent Subnets} 

\begin{figure*} [!t]
  \centering
    \includegraphics[width=0.5\linewidth]{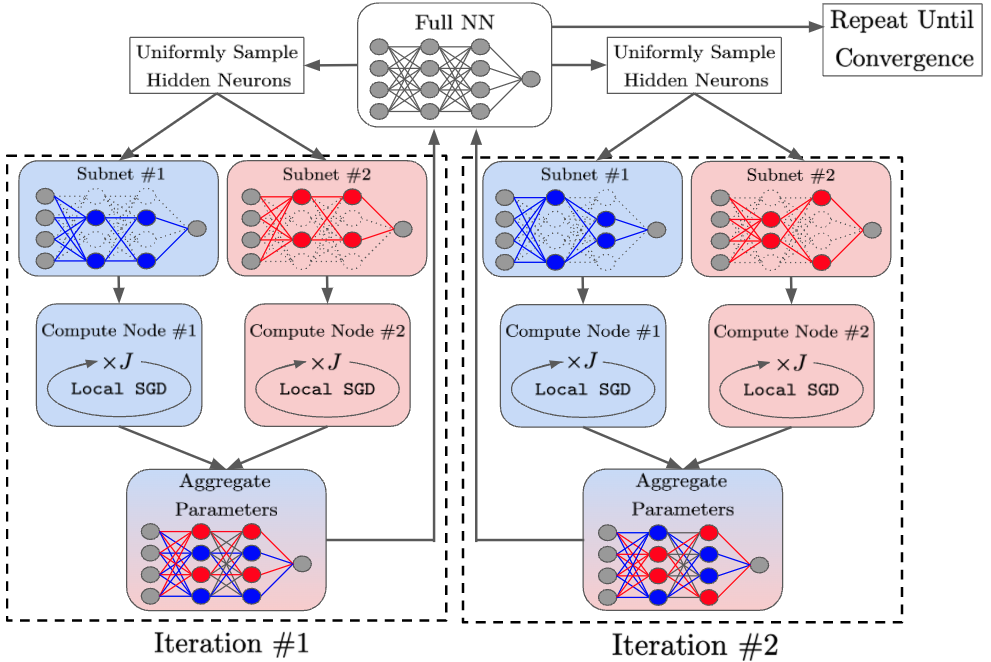}
    \caption{Schematic depiction of a two-hidden-layer NN being trained with IST across two nodes. Each layer's neurons are partitioned randomly to a single node, excluding input and output neurons (i.e., these are shared between sites). The first two iterations of IST are depicted, but the same process of sampling, training, and aggregating subnets repeats until convergence.}
  \label{fig:ist_depict}
\end{figure*}

\subsection{Methodology} \label{S:methodology}
\noindent
\textbf{Notation.}
Assume $n$ sites in a distributed system.
Let $f_l$ denote that vector of activations at layer $l$. 
$f_t$ denotes the set of activations at the final or ``top'' layer of the network, and $f_0$ denotes the feature vector that is input into the network. 
Assume that the number of neurons at layer $l$ is $N_l$.  
Let $\ell(w, \cdot)$ denote the loss function of a NN with parameter $w$.
Given samples $X:=\{x_i, y_i\}_{i = 1}^q$, we aim to find a $w^\star$ that minimizes the \emph{empirical loss} over a set of labeled examples:
\begin{align}
w^\star \in \arg\min_{w} \frac{1}{q}\sum_{i = 1}^q \ell\left(w, \left\{x_i, y_i\right\} \right). \label{eq:loss}
\end{align}
Although \eqref{eq:loss} can be solved in numerous ways  \cite{wright1999numerical, zeiler2012adadelta, kingma2014adam, duchi2011adaptive, ruder2016overview}, nearly all NN training is accomplished using some variant of SGD.
Here, $\eta > 0$ is the learning rate and $i_t$ denotes a subset of training examples from $X$.

\noindent \textbf{Constructing Subnets.}
IST is a distributed training regime that randomly partitions hidden neurons via uniform assignment to one of $n$ possible compute nodes.
Neurons assigned to the same compute node form a ``subnet''.
Then, the weights of the full NN are partitioned accordingly based on the active neurons for each subnet.
Hidden neurons are assigned to exactly one subnet to ensure that $i)$ all neurons are included in training and $ii)$ the same neuron is not simultaneously partitioned to multiple subnets.

Subnet construction is depicted in Figure \ref{fig:ist_depict} for a two-hidden-layer NN distributed across $n=2$ compute nodes.
Input and output layers are fully utilized at all sites.
Notably, certain parameters are not partitioned to any subnet within multi-layer NNs.\footnote{Parameters are only active within a subnet if both input and output neurons associated with that parameter are sampled in the same subnet. For NNs with a single hidden layer, all parameters are included in some subnet because input and output neurons are shared across compute nodes.}
However, in contrast to ensemble-style techniques, IST randomly samples new subnets frequently throughout the learning process, ensuring that all parameters have a high likelihood of being trained sufficiently after several subnet groups have been sampled.

Subnet output is computed by masking (i.e., setting to zero) inactive neurons and scaling remaining activations by a factor of $n^2$ (i.e., to counteract neuron removal with uniform probability $\frac{1}{n}$).
Such a forward pass, which is formalized in Appendix \ref{A:unbiased}, provides an unbiased estimate of the full NN forward pass (excluding activation functions).
Furthermore, performing uniform sampling of neurons independently at each layer yields sublinear convergence to a first-order stationary point on expectation; see Appendix \ref{A:ist_conv}.
Thus, we adopt this uniform sampling policy in IST due to its unbiased nature and rigorous theoretical guarantees. 

\noindent\textbf{Distributing Subnets.}
The $n$ subnets produced by IST are disjoint, meaning that no model parameters are simultaneously partitioned to multiple subnets.
As a result, when distributed to a separate compute node, subnets $i)$ require no cross-site communication during their forward pass, and $ii)$ only require a fraction $\frac{1}{n^2}$ of layer parameters to be sent to each compute node. 
Thus, \emph{subnets can be distributed to separate compute nodes without significant communication overhead and trained with no dependence upon other subnets---an approach that is adopted directly within IST}.
	
\noindent \textbf{Training Subnets.}
For training, IST sends each of the $n$ subnets to a separate compute node and performs $J$ iterations of local SGD \cite{lin2018don}.
After such independent training iterations, subnet parameters are copied back into the full NN, where no collisions occur because the parameter partition is disjoint. 
Then, a new group of subnets is constructed through random sampling (i.e., a ``re-sampling'' of network parameters) and the process repeats.

Unlike ensemble methods that independently train each subnet and aggregate parameters into the full NN once at the end of training, IST re-samples subnets frequently and trains them for a shorter number of iterations between re-samplings.
Such re-sampling is necessary to avoid the accumulation of random effects, as the expected input to a neuron---despite being unbiased---will shift after backpropagation.
Such a shift may be inconsistent across sites, because subnets are trained on data samples from the same distribution; but, re-sampling---which is not present within ensemble methods---guards against such an occurrence.

\subsection{Additional Considerations} \label{S:add_consider}

\noindent \textbf{Correcting Distributional Shift.}
The analysis of the unbiased subnet forward pass in Appendix \ref{A:unbiased} does not consider the NN's non-linear activation function.\footnote{$\mathbb{E}[x] = b$ does not imply that  $\mathbb{E}[f(x)] \approx f(b)$ for some random variable $x$ when $f$ is non-linear.}
Within IST, the inputs to each subnet neuron are sub-sampled and scaled by a factor of $n^2$ to unbias the neuron's activation, which increases the standard deviation of the input to each neuron by a factor of $n$.
As a result, extreme input values are more likely to be observed during training (i.e., when using subnets) than during deployment. 
To correct this distributional mismatch, we remove the $n^2$ correction factor and instead compute the mean $\mu$ and standard deviation $\sigma$ of the inputs to each neuron during training and transform subnet output as $x = (x - \mu)/\sigma$ before passing it through the non-linear activation function.
After training is complete, we compute $\mu$ and $\sigma$ for each neuron over a small subset of training data using the full network---these values can then be used when the network is deployed.

Although this approach is similar to batch normalization \cite{ioffe2015batch}, the motivation for its use is much different. 
Namely, while batch normalization maintains a non-saturated range of neuron input during training to accelerate convergence and improve generalization, IST will not work in the absence of such normalization.
The distributional shift encountered when deploying the network must be corrected, making this modification an essential component of IST, rather than an aid to model training and performance.

\noindent \textbf{Other Architectures.}
IST can be extended to common network architectures (e.g., convolutional NNs) by applying IST only to fully-connected layer(s) within the network (i.e., these exist within most modern convolutional NN architectures).
Here, the fully-connected layers would be decomposed as described previously, while the rest of the network is broadcast to every site during training.
Such an approach has significant benefits, as fully-connected layers tend to contain a large portion of network parameters.\footnote{Consider the full ImageNet dataset \cite{deng2009imagenet, krizhevsky2012imagenet} for a deep model like ResNet50: the convolutional layers have 17,614,016 parameters (67.2MB, $28.2\%$), whereas the fully-connected layer has 44,730,368 parameters (170.6MB, $71.8\%$) that utilized in IST.}
Thus, improving the efficiency of distributed training over fully-connected layers benefits the distributed training process for the entire network. 

\subsection{Analysis}
IST reduces communication overhead in comparison to data parallel training approaches, which broadcast all parameters across sites during each round of training.
Measuring the inflow to each site, the total network traffic of data parallel training per gradient step is (in floating point numbers transferred):
\begin{align*}
    \sum_{i = 1}^t nN_{i-1}N_i. \\[-18pt] \nonumber
\end{align*}
In contrast, in IST, each site receives current parameters every $J$ gradient steps (i.e., assuming $J$ iterations of local SGD are performed between re-sampling rounds).
Furthermore, subsampling the NN into multiple subnets further reduces the communication cost of IST because input/ouput layers are partitioned (not broadcast) across nodes and each node receives only a $\frac{1}{n}$ ratio of other network parameters. 
The total network traffic of IST per gradient step is:
\begin{align*}
\frac{N_0 N_1 + N_{t-1}N_t}{J} + \sum_{i = 1}^l \frac{N_{i-1}N_i}{n \times J}. \\[-18pt] \nonumber    
\end{align*}
 
Similarly, IST reduces computational resource utilization in comparison to data parallel.
Given the FLOPs required by matrix multiplications during forward/backward steps, in  ``classical'' data parallel training, the number of FLOPS required per gradient step is: 
\begin{align*}
    4 \sum_{i = 1}^l BN_{i-1}N_i.\\[-18pt] \nonumber
\end{align*}
In contrast, the number of FLOPS gradient step within IST is: 
\begin{align*}
    4BN_0 N_1 + 4BN_{t-1}N_t
+ 4B\sum_{i = 1}^l \frac{N_{i-1}N_i}{n}. \\[-18pt] \nonumber    
\end{align*}

Note that this computational reduction indicates that training models with IST reduces memory requirements, which enables the training of larger models as shown in Sec. \ref{S:experiments}.

\begin{table*}[h]
\centering
\begin{small}
\begin{tabular}{lclclclclclclclclcl}
\toprule
\multicolumn{19}{c}{Google Speech 2 Layer} \\
\midrule
& & \multicolumn{5}{c}{Data Parallel} & & \multicolumn{5}{c}{Local SGD} & & \multicolumn{5}{c}{IST} \\ \cmidrule{3-7} \cmidrule{9-13} \cmidrule{15-19}
Accuracy & & 2 Node & & 4 Node & & 8 Node & & 2 Node & & 4 Node & & 8 Node & &  2 Node & & 4 Node & & 8 Node \\
\cmidrule{1-1} \cmidrule{3-3} \cmidrule{5-5} \cmidrule{7-7} \cmidrule{9-9} \cmidrule{11-11} \cmidrule{13-13} \cmidrule{15-15} \cmidrule{17-17} \cmidrule{19-19}
0.63 & & 118 & & 269 & & 450 & & 68 & & 130 & & 235 & & 35 & & 28 & & \textcolor{blue}{\textbf{24}} \\
0.75 & & 759 & & 1708 & & 2417 & & 444 & & 742 & & 1110 & & 231 & & \textcolor{blue}{\textbf{167}} & & 192 \\ 
\midrule
\multicolumn{19}{c}{Google Speech 3 Layer } \\
\midrule
& & \multicolumn{5}{c}{Data Parallel} & & \multicolumn{5}{c}{Local SGD} & & \multicolumn{5}{c}{IST} \\ \cmidrule{3-7} \cmidrule{9-13} \cmidrule{15-19}
Accuracy & & 2 Node & & 4 Node & & 8 Node & & 2 Node & & 4 Node & & 8 Node & &  2 Node & & 4 Node & & 8 Node \\
\cmidrule{1-1} \cmidrule{3-3} \cmidrule{5-5} \cmidrule{7-7} \cmidrule{9-9} \cmidrule{11-11} \cmidrule{13-13} \cmidrule{15-15} \cmidrule{17-17} \cmidrule{19-19}
0.63 & & 376 & & 1228 & & 1922 & & 182 & & 586 & & 1115 & & \textcolor{blue}{\textbf{76}} & & 141 & & 300 \\
0.75 & & 4534 & & 9340 & & 14886 & & 2032 & & 4107 & & 6539 & & 812 & & \textcolor{blue}{\textbf{664}} & & 1161 \\ 
\midrule
\multicolumn{19}{c}{CIFAR10 Resnet18} \\
\midrule
& & \multicolumn{5}{c}{Data Parallel} & & \multicolumn{5}{c}{Local SGD} & & \multicolumn{5}{c}{IST} \\ \cmidrule{3-7} \cmidrule{9-13} \cmidrule{15-19}
Accuracy & & 2 Node & & 4 Node & & 8 Node & & 2 Node & & 4 Node & & 8 Node & &  2 Node & & 4 Node & & 8 Node \\
\cmidrule{1-1} \cmidrule{3-3} \cmidrule{5-5} \cmidrule{7-7} \cmidrule{9-9} \cmidrule{11-11} \cmidrule{13-13} \cmidrule{15-15} \cmidrule{17-17} \cmidrule{19-19}
0.85 & & 21775 & & 13689 & & 6890 & & 18769 & & 12744 & & 7020 & & 15093 & & 7852 & & \textcolor{blue}{\textbf{5241}} \\
0.90 & & 54002 & & 38430 & & 17853 & & 36891 & & 22198 & & \textcolor{blue}{\textbf{12157}} & & 33345 & & 16798 & & 13425 \\ 
\midrule
\multicolumn{19}{c}{Full ImageNet VGG12} \\
\midrule
& & \multicolumn{5}{c}{Data Parallel} & & \multicolumn{5}{c}{Local SGD} & & \multicolumn{5}{c}{IST} \\ \cmidrule{3-7} \cmidrule{9-13} \cmidrule{15-19}
Accuracy & & 2 Node & & 4 Node & & 8 Node & & 2 Node & & 4 Node & & 8 Node & &  2 Node & & 4 Node & & 8 Node \\
\cmidrule{1-1} \cmidrule{3-3} \cmidrule{5-5} \cmidrule{7-7} \cmidrule{9-9} \cmidrule{11-11} \cmidrule{13-13} \cmidrule{15-15} \cmidrule{17-17} \cmidrule{19-19}
0.20 & & 108040 & & 278542 & & 504805 & & 6900 & & 14698 & & 30441 & & \textcolor{blue}{\textbf{3629}} & & 4379 & & 5954 \\
0.26 & & 225911 & & 393279 & & 637188 & & 15053 & & 22055 & & 39439 & & \textcolor{blue}{\textbf{6189}} & & 7711 & & 10622 \\
\bottomrule
\end{tabular}
\caption {The time (in seconds) to reach various levels of accuracy.}
\label{tab:converge}
\end{small}
\end{table*}

\noindent \textbf{Convergence Guarantees.}
We show that the IST decomposition guarantees convergence to a first-order stationary point on expectation in the distributed setting.
Namely, under common assumptions of smoothness, Lipschitz continuity of the objective, and stochastic error boundedness, IST converges sublinearly to a bounded error region around a stationary point; see Appendix \ref{A:ist_conv}.

\section{Empirical Evaluation} \label{S:experiments}

In this section, we design a set of experiments that showcase the potential benefits of the IST approach under cases of mandatory distribution with limited hardware capabilities.
In these cases, we assume slow networks connections with CPUs or GPUs (possibly with limited memory) available on each node.
We consider a wide variety of learning tasks and network architectures (i.e., both fully-connected NNs and more complex NNs that contain fully-connected layers).

\subsection{Setup and Details}

As previously mentioned, model parallel training is not appropriate for mandatory distribution, as all needed data must be stored on the node that houses the network's input module. 
Furthermore, due to the assumption of limited memory on each device, our experiments typically consider shallow, fully-connected NNs with wide hidden layers.\footnote{The memory usage of IST scales linearly with increasing network depth, but small portions of each hidden layer can be partitioned to each subnet in order to limit increased memory usage due to larger hidden layers.}
Popular model parallel training packages (e.g., Gpipe \cite{huang2019gpipe}) struggle to perform well on wide models with few layers, providing further evidence that model parallel training is not the proper training approach when distribution is mandatory.
As such, we adopt local SGD \cite{lin2018don}, data parallel training, and ensemble learning as our major experimental baselines.

\noindent
\textbf{Experimental Settings.} We consider the following scenarios: 
\begin{itemize}[leftmargin=0.5cm]
    \item  \textit{Google Speech Commands} \cite{warden2018speech}: We learn a 2-layer network of 4096 neurons and a 3-layer network of 8192 neurons to recognize 35 labeled keywords from audio waveforms (in contrast to the 12 keywords in prior reports \cite{warden2018speech}). We represent each waveform as a 4096-dimensional feature vector \cite{stevens1937scale}. 
    \item \textit{Image classification on CIFAR10 and full ImageNet} \cite{he2016deep, simonyan2014very}: We train the Resnet18 model over CIFAR10, and the VGG12 model over full ImageNet (see Section \ref{S:add_consider} for a discussion of IST and non-fully connected architectures). \textbf{Note that we include the complete ImageNet dataset with all $21,841$ categories and report the top-10 accuracy} \cite{kolesnikov2020big, ridnik2021imagenet, dosovitskiy2020image}. 
    \item \textit{Amazon-670k} \cite{bhatia2019the}: We train a 2-layer, fully-connected neural network, which accepts a $135,909$-dimensional input feature, and generates a prediction over $670,091$ output labels.
\end{itemize}

We train Google speech and Resnet18 on CIFAR10 on three AWS CPU clusters, with 2, 4, and 8 CPU instances (\texttt{m5.2xlarge}).  We train the VGG model on full ImageNet and Amazon-670k extreme classification network on three AWS GPU clusters, with 2, 4, and 8 GPU machines (\texttt{p3.2xlarge}).
Our choice of AWS was deliberate, as it is a common platform for distributed training and presents the common challenge faced by many consumers---distributed ML without a super-fast interconnect.

\begin{figure*}[h]
    \centering
    \begin{subfigure}
        \centering
        \includegraphics[width=0.24\textwidth]{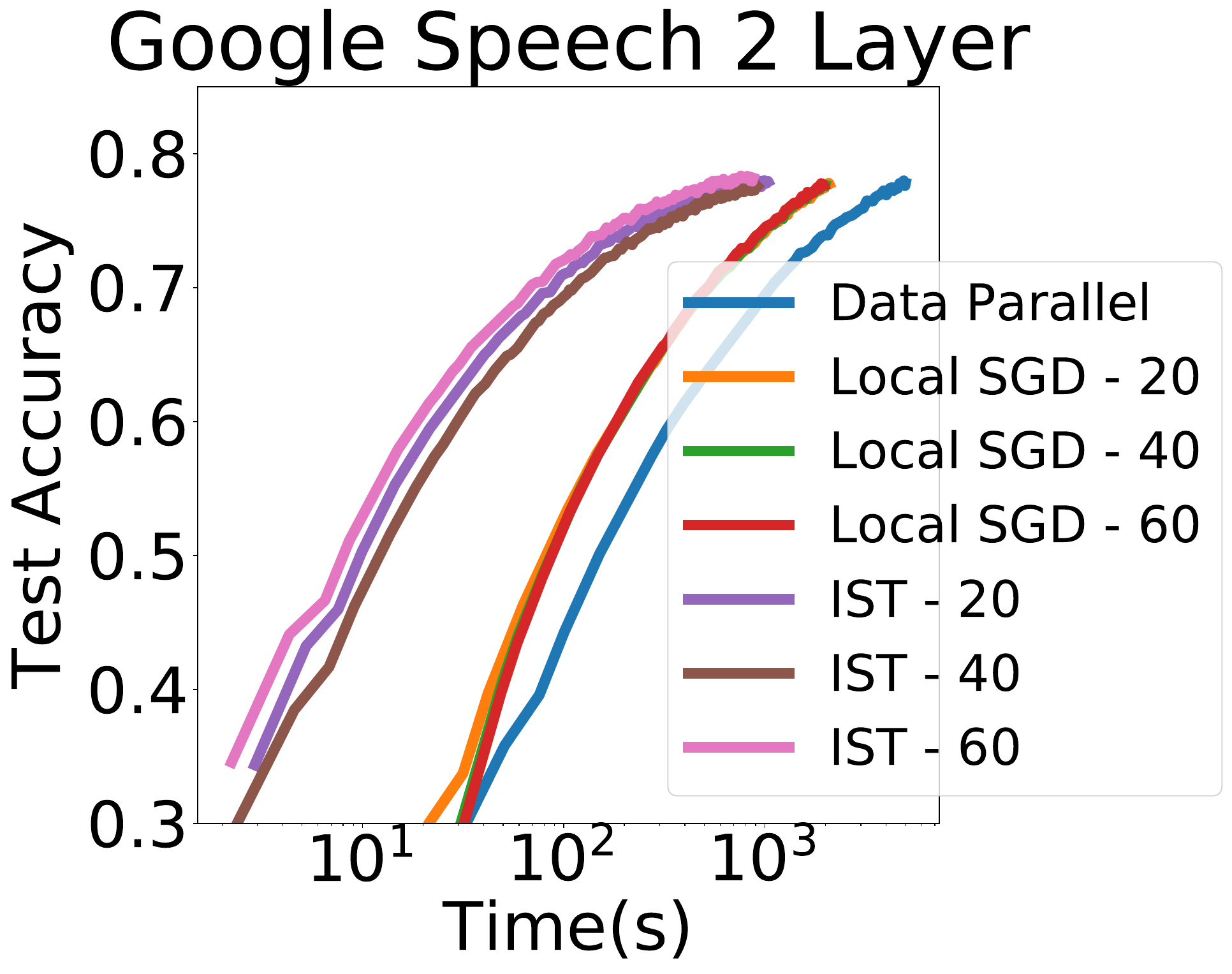}
    \end{subfigure}
    \begin{subfigure}
        \centering
        \includegraphics[width=0.24\textwidth]{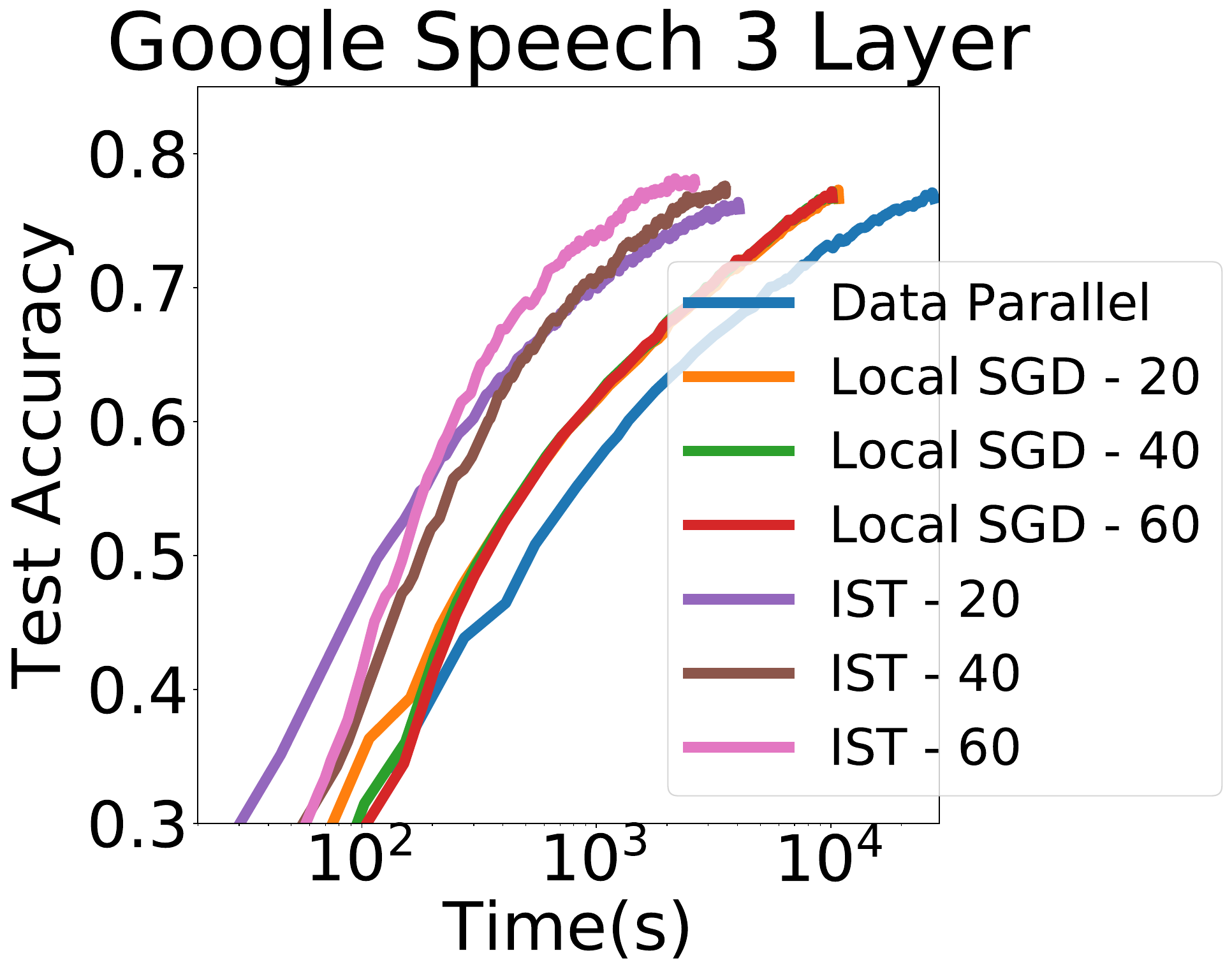}
    \end{subfigure}
    \begin{subfigure}
        \centering
        \includegraphics[width=0.24\textwidth]{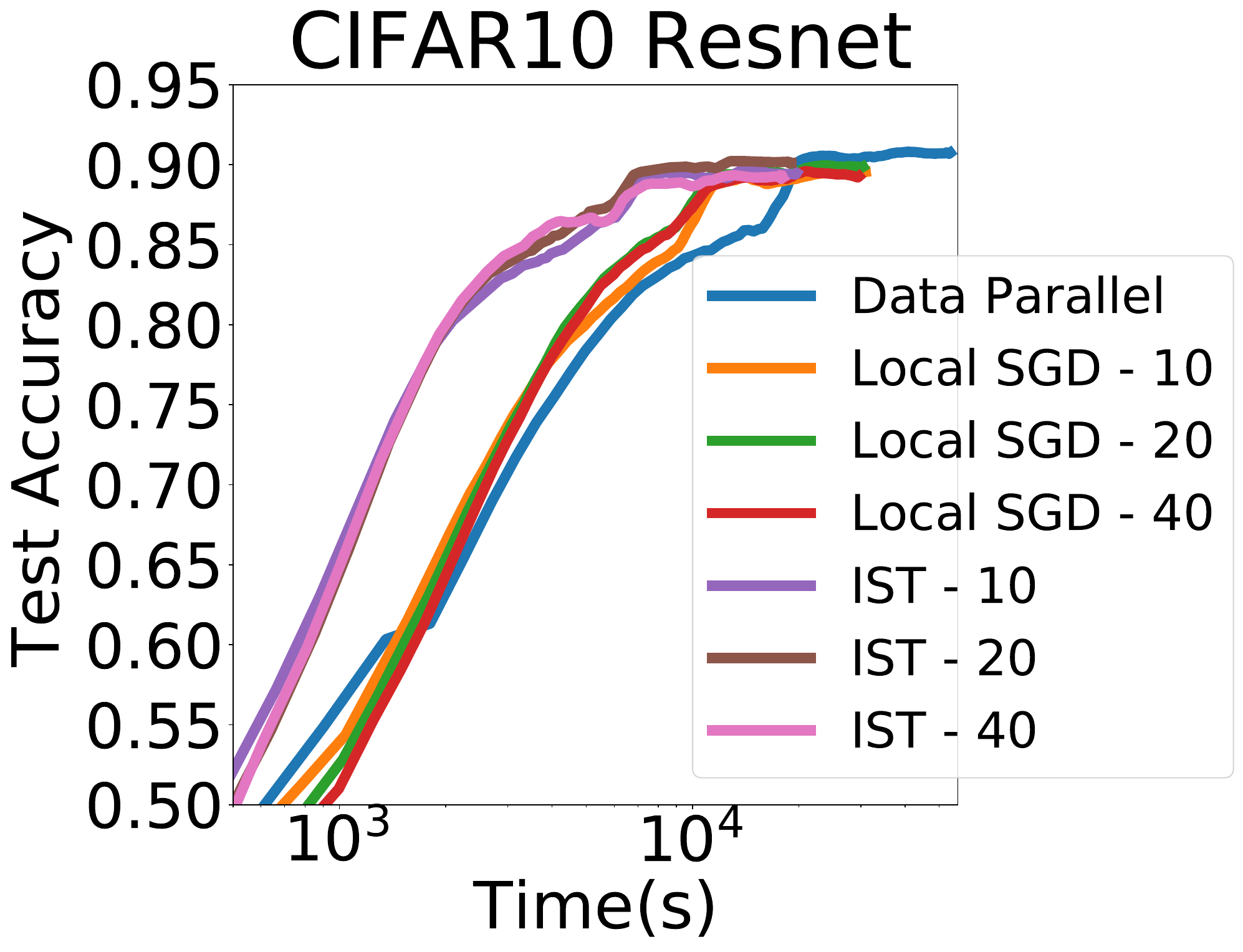}
    \end{subfigure}
    \begin{subfigure}
        \centering
        \includegraphics[width=0.24\textwidth]{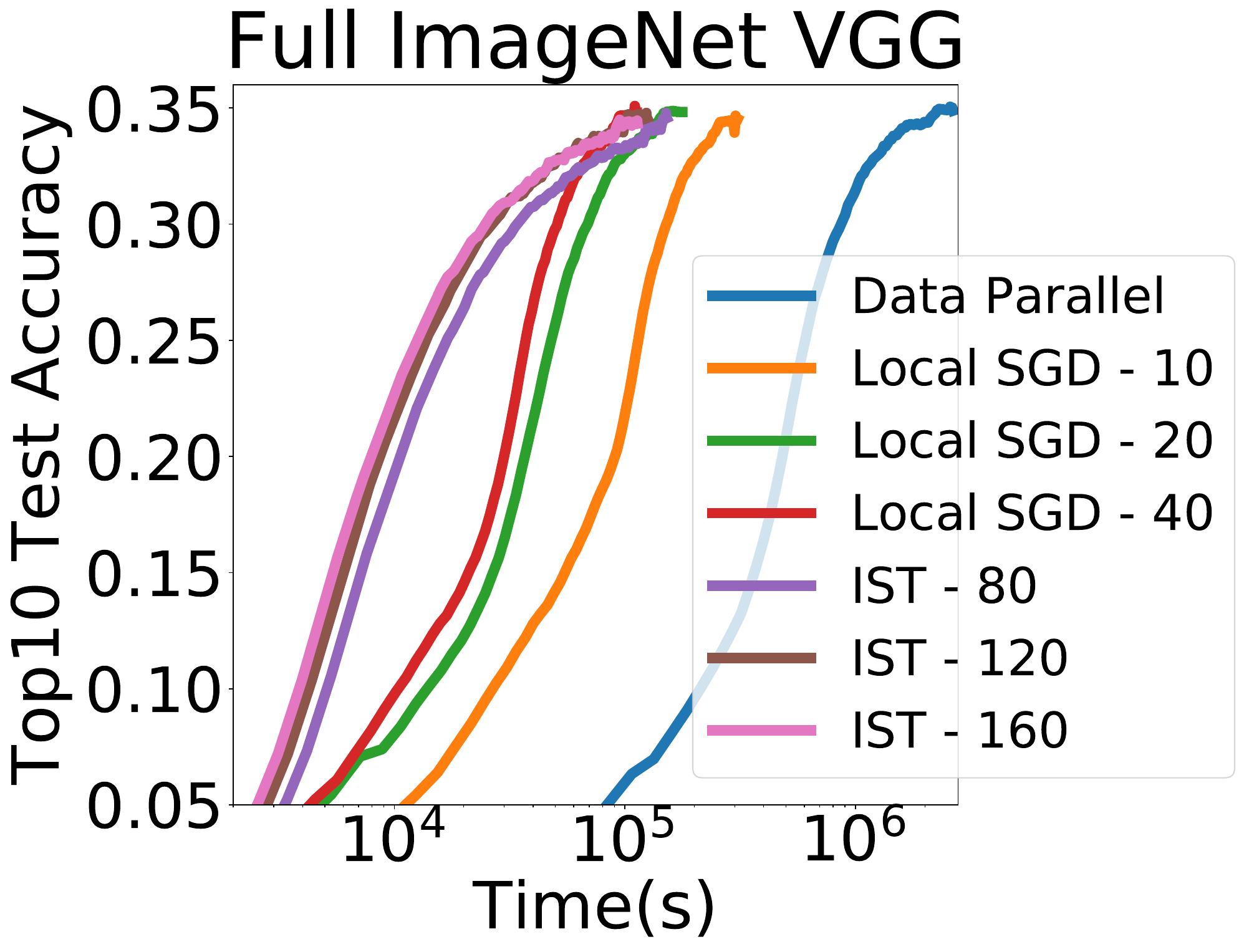}
    \end{subfigure} 
    \caption{Test accuracy vs. time. 2-/3- layer Google speech models are trained using an 8-CPU cluster; Resnet18 on CIFAR10 is trained using 4-CPU cluster; VGG12 on full ImageNet is trained using a 8-GPU cluster. The number after local SGD or IST legend represents the local update iterations.}
    \label{exp_convergence}
\end{figure*}

\noindent
\textbf{Distributed Implementation Notes.} 
We implement a distributed parameter server for IST in PyTorch.
We compare IST to the PyTorch implementation of data parallel learning. 
We also adapt the PyTorch data parallel learning to realize local SGD \cite{lin2018don} and ensemble learning, where learning occurs locally for a number of iterations before synchronizing.
For ensemble learning, this synchronization only occurs once at the end of training. 

For the CPU experiments, we use PyTorch's \texttt{gloo} backend. 
For the GPU experiments, data parallel learning and local SGD use PyTorch's \texttt{nccl} backend, which leverages the most advanced Nvidia collective communication library (the set of high-performance multi-GPU and multi-node collective communication primitives optimized for NVIDIA GPUs). 
\texttt{Nccl} implements ring-based all-reduce \cite{xu2018nccl}, which is used in well-known distributed learning systems such as Horovod \cite{sergeev2018horovod}. 

IST cannot use the \texttt{nccl} backend because it does not support the \texttt{scatter} operator required to implement IST. 
As a result, IST must use the \texttt{gloo} backend (meant for CPU-based learning), which is a serious handicap but does not reflect any intrinsic flaw of the method---high-performance GPU libraries simply lack support for required operations.

\subsection{Results and Analysis} \label{S:analysis}
\noindent
\textbf{Convergence speed.} While IST can process data quickly (i.e., due to previously-described improvements in communication efficiency),  there are questions regarding its statistical efficiency and generalization performance in comparison to baseline methods.
Figure \ref{exp_convergence} plots the hold-out test accuracy for selected benchmarks as a function of time, while Table \ref{tab:converge} shows the training time required for IST and relevant baselines to reach specified levels of hold-out test accuracy.

Our results generally indicate that IST achieves high accuracy on the test set much faster than other frameworks.
For example, in reaching an accuracy of 77\% with a 2-layer, fully-connected network, IST exhibits a $4.2 \times$ speedup compared to local SGD and a $10.6 \times$ speedup compared to data parallel.
Similarly, IST exhibits a $6.1 \times$ speedup compared to local SGD and a $16.6 \times$ speedup compared to data parallel in reaching the same accuracy with a 3-layer model.  
Note that the above improvements were observed even though IST was handicapped by its use of the \texttt{gloo} backend for its GPU implementation. 

Because CPUs were used for training on CIFAR10, the network was less of a bottleneck and all methods were able to scale, thus slightly negating the advantages of IST. 
Despite reaching 90\% accuracy slower on an 8-CPU cluster, however, IST was still the fastest to reach 85\% accuracy.
Furthermore, for the full ImageNet data set, the communication bottleneck using AWS is so severe that the smaller clusters were always faster. 
At each cluster size, IST was still the fastest option.

\begin{table}[H]
    \centering
    \begin{small}
    \begin{tabular}{clclclcl}
        \toprule
        & & Data Parallel & & Local SGD & & IST \\
        \cmidrule{3-3} \cmidrule{5-5} \cmidrule{7-7}
        Speech 2 layer & & 0.7938 & & 0.7998 & & \textcolor{blue}{\textbf{0.8153}} \\
        Speech 3 layer & & 0.7950 & & 0.7992 & & \textcolor{blue}{\textbf{0.8327}} \\
        CIFAR10 & & \textcolor{blue}{\textbf{0.9128}} & & 0.9087 & & 0.9102 \\   
        Full Imagenet & & 0.3688 & & 0.3685 & & \textcolor{blue}{\textbf{0.3802}} \\
        \bottomrule
    \end{tabular}
    \end{small}
    \caption {Final accuracy on each benchmark.}
    \label{tab:accuracy} 
\end{table}
\begin{table}[]
    \centering
    \begin{small}
    \begin{tabular}{cccccc}
    \toprule
       \multirow{2}{*}{Compute Nodes} &\multicolumn{2}{c}{2 Layer} && \multicolumn{2}{c}{3 Layer}\\
       \cmidrule{2-3} \cmidrule{5-6}
        & IST & Ensemble && IST & Ensemble \\
        \midrule
        2 Node & 0.82 & 0.82 && 0.84 & 0.74 \\
        4 Node & 0.79 & 0.80 && 0.82 & 0.71 \\
        8 Node & 0.76 & 0.77 && 0.78 & 0.70 \\\bottomrule
    \end{tabular}
    \caption{Test accuracy on Google Speech Commands.} 
    \label{tab:ensemble} 
    \end{small}
\end{table}

\noindent \textbf{Trained model accuracy.}
In Table \ref{tab:accuracy} we give the final accuracy of each method, trained on a 2-node cluster in various experimental settings.
As can be seen, despite partitioning the full networks into several independently-trained subnets, IST achieves better final accuracy in comparison to data parallel and local SGD training on all datasets except for CIFAR10.
On the CIFAR10 dataset, IST achieves test accuracy 0.26\% lower than data parallel training, but outpeforms local SGD.
Furthermore, IST continues to achieve high final accuracy as the number of compute nodes is increased, as shown in Table \ref{tab:ensemble}.
Thus, IST achieves highly-comparable or improved final accuracy in comparison to local SGD, data parallel, and ensemble-based training in all settings, \emph{revealing that the partitioning strategy of IST does not deteriorate the network's ability to match or exceed the accuracy achieved by baseline methodologies.}

As previously mentioned, IST also enables models to be trained that are too large to be handled by a single device.
In cases of mandatory distribution, such a property is useful for training sufficiently large models despite limited memory on individual compute nodes.
To demonstrate the utility of this property of IST, we study the relationship between embedding dimension and test accuracy for fully-connected models trained on the Amazon-670K recommendation task in an 8-GPU cluster.
As shown in Table \ref{tab:extreme}, IST is able to scale to larger model sizes without exceeding the memory capacity of individual nodes.
Such scalability enables a $>15\%$ test accuracy improvement in comparison to data parallel training, thus demonstrating that IST allows models with sufficient capacity to be trained despite the restricted memory of each device.


\noindent 
\textbf{Comparison to Ensemble Learning.}
IST intermittently aggregates subnet parameters and re-samples a new group of subnets for independent training. 
Although ensemble learning trivially improves communication efficiency and wall-clock training time (i.e., due to utilizing fewer synchronizations), re-sampling is necessary for achieving high network performance.
To show this, we perform tests with ensemble learning---i.e., training a group of subnet-sized models independently and aggregating their parameters once at the end of training---and IST on the Google Speech Commands dataset; see Table \ref{tab:ensemble}.

Ensemble learning and IST perform similarly for two-layer networks \ref{S:add_consider}.
Such comparable performance is expected because $i)$ two-layer networks are separable \cite{chizat2018global} (i.e., each hidden neuron contributes independently to network output without inter-neuron interaction) and $ii)$ all parameters within the one-hidden-layer NN are partitioned to some subnet. 
In such a simplified case, ensemble learning is able to performing well by independently learning meaningful neuron representations that can be aggregated into the global network. 

For deeper networks, no re-sampling during training leads numerous network parameters to be excluded from the learning process and allows random effects to accumulate throughout training, thus drastically deteriorating ensemble learning performance.
As such, IST significantly outperforms ensemble learning with three-layer networks (e.g., 10\% absolute improvement with $n=2$ compute nodes), revealing that IST has a significant performance advantage relative to ensemble learning for complex, multi-layer network architectures. 
Thus, although ensemble learning is faster to complete a fixed number of training epochs, it cannot yield comparable performance to networks trained with IST. 

\noindent
\textbf{Discussion.} There are a few takeaways from the experimental results. First, as expected, IST is able to process far more data in a short amount of time than the other distributed training algorithm.
Interestingly, we find that the IST speedups in CPU clusters are more significant than that in GPU clusters. There are two reasons for this. First, for GPU clusters, IST suffers from its use of PyTorch's \texttt{gloo} backend, compared to the \texttt{all-reduce} operator provided by \texttt{nccl}. Second, since the GPU provides a very high level of computation, there is less benefit to be realized from the reduction in FLOPS per gradient step using IST (as the GPU does not appear to be compute bound).

It is interesting that \emph{some of the frameworks actually do worse with additional machines, especially with a fast GPU}. This illustrates a significant problem with distributed learning. Unless a super-fast interconnect is used (and such interconnects are not available from typical cloud providers), it can actually be \emph{detrimental} to add additional machines, as the added cost of transferring data can actually result in \emph{slower} running times.  We see this clearly in Table \ref{tab:converge}, where the state-of-the-art PyTorch data parallel implementation (and the local SGD variant) does significantly \emph{worse} with more machines. IST shows the best potential to utilize additional machines without actually becoming much slower or slower to reach high accuracy.  

Finally, we note that various compression schemes can be used to increase the bandwidth of the interconnect (e.g., gradient sparsification \cite{aji2017sparse}, quantization \cite{alistarh2017qsgd}, sketching \cite{ivkin2019communication}, and low-rank compression \cite{vogels2019powersgd}).  However, these methods could be used with \emph{any} framework---including IST. 
We conjecture that while compression may allow effective scaling to larger clusters, it would not affect the efficacy of IST.

\begin{table}[H]
    \centering
    \begin{footnotesize}
    \begin{tabular}{ccccccccccccc}
        \toprule
        & & \multicolumn{5}{c}{Data Parallel} & &  \multicolumn{5}{c}{IST} \\ 
        \cmidrule{3-7} \cmidrule{9-13}
        Dim. & & @1 & & @3 & & @5  & & @1 & & @3 & & @5    \\
        \cmidrule{1-1} \cmidrule{3-3} \cmidrule{5-5} \cmidrule{7-7} \cmidrule{9-9} \cmidrule{11-11} \cmidrule{13-13}
        512 & & 0.386 & & 0.345 & & 0.316 & & 0.396 & & 0.360 & & 0.331 \\
        1024 & & \multicolumn{5}{c}{Fail to handle} & & 0.409 & & 0.369 & & 0.339 \\
        1536 & & \multicolumn{5}{c}{Fail to handle} & & 0.432 & & 0.391 & & 0.361 \\ 
        2048 & & \multicolumn{5}{c}{Fail to handle} & & 0.437 & & 0.394 & & 0.364 \\ 
        2560 & & \multicolumn{5}{c}{Fail to handle} & & \textcolor{blue}{\textbf{0.438}} & & \textcolor{blue}{\textbf{0.394}} & & \textcolor{blue}{\textbf{0.366}} \\
        \bottomrule
    \end{tabular}
     \end{footnotesize}
    \caption {Precision @1, @3, @5 on Amazon 670k.} \label{tab:extreme}
\end{table}

\section{Related work}{\label{sec:related}}

Data parallelism often suffers from the high bandwidth costs to communicate gradient updates between workers. 
Quantized SGD \cite{alistarh2017qsgd, courbariaux2015binaryconnect, seide2014bit, dettmers2015bit, gupta2015deep, hubara2017quantized, wen2017terngrad} and sparsified SGD \cite{aji2017sparse} both address this.
Quantized SGD uses lossy compression to quantize the gradients. Sparsified SGD reduces the exchange overhead by transmitting the gradients with maximal magnitude. Such methods are orthogonal to IST, and could be used in combination with it.

Recently, there has been a series of papers on using parallelism to \emph{``Solve the ${\tt YY}$ learning problem in ${\tt XX}$ minutes''}, for ever-decreasing values of ${\tt XX}$ \cite{goyal2017accurate, yadan2013multi, you2017scaling, smith2017don, codreanu2017scale, you2019reducing, you2019large}. Often these methods employ large batches.
It is generally accepted---though still debated \cite{dinh2017sharp}---that large batch training converges to ``sharp minima'', hurting generalization \cite{keskar2016large, yao2018hessian, defazio2018ineffectiveness}.
Further, achieving such results seems to require teams of PhDs utilizing special-purpose hardware: there is no approach that generalizes well without extensive trial-and-error.

Distributed local SGD \cite{mcdonald2009efficient, zinkevich2010parallelized, zhang2014dimmwitted, zhang2016parallel} updates the parameters, through averaging, only after several local steps are performed per compute node. 
This reduces synchronization and thus allows for higher hardware efficiency \cite{zhang2016parallel}.
IST uses a similar approach but makes the local SGD and each synchronization round less expensive.
Recent approaches \cite{lin2018don} propose less frequent synchronization towards the end of the training, but they cannot avoid it at the beginning.

Finally, asynchrony avoids SGD synchronization cost \cite{recht2011hogwild, dean2012large, paine2013gpu, zhang2013asynchronous}.
It has been used in distributed-memory systems, such as DistBelief \cite{dean2012large} and the Project Adam \cite{kingma2014adam}. 
While such systems, asymptotically, show nice convergence rate guarantees, there seems to be growing agreement that \emph{unconstrained} asynchrony does not always work well \cite{chen2016revisiting}, and it seems to be losing favor in practice.

Overall, the goal of such distributed training methods is to lower the wall-clock time-to-convergence with the addition of extra hardware.
As such, empirical analysis of these methods is often conducted using state-of-the-art computing hardware with high-bandwidth interconnects.
Even with access to such an ideal environment, however, data parallel approaches struggle to scale.
In particular, per-node compute requirements are reduced while synchronization costs remain constant or increase, leading to cases where the addition of more nodes makes training \emph{slower} as communication costs begin to dominate the training procedure. 
This issue could theoretically be mitigated with the use of larger batches, but such an approach often degrades statistical efficiency and leads to poor generalization \cite{ma2019inefficiency, golmant2018computational, goyal2017accurate, yadan2013multi, you2017scaling, smith2017don, codreanu2017scale, you2019reducing, you2019large}. 

\section{Conclusion}{\label{sec:conclusion}}

In this work, we propose \emph{independent subnet training} for distributed training of neural networks. 
By stochastically partitioning the model into non-overlapping subnets, 
IST reduces the communication overhead for model synchronization, and the computation workload of forward-backward propagation for a thinner model on each worker. This results in two advances: $i)$ IST significantly accelerates the training process comparing with standard data parallel approaches for distributed learning, and $ii)$ IST scales to large models that cannot be learned using standard data parallel approaches.

\appendix

\section{IST is an Unbiased Estimator} \label{A:unbiased}
We formalize subnet construction in IST with a set of neuron membership indicators $m^{(s)}_{l,i} \in \{0, 1\}$ at each layer $l$ where $s$ ranges over the $n$ compute nodes and $i$ ranges over all the neurons in layer $l$. 
$m^{(s)}$ contains a binary mask for subnet $s$ across all layers and neurons.
For each entry $m^{(s)}_{l, i}$, a value of $\{0, 1\}$ is assigned with marginal probability $\mathbb{P}[m^{(s)}_{l,i} = 1] = \frac{1}{n}$ to exactly one of the $n$ subnets, implying that $\sum_s m^{(s)}_{l,i} = 1$ (i.e., neurons are partitioned to exactly one subnet) and $\mathbb{E}[m^{(s)}_{l,i'} m^{(s)}_{l-1,i}] = \frac{1}{n^2}$ (i.e., sampling is independent at each layer).

Using these constructions, we can define the forward pass of subnet $s$ at layer $l$ as
\begin{equation}
\hat{f}^{(s)}_l = n^2 \left( m^{(s)}_l \odot \left (W_{l} \left( m^{(s)}_{l-1} \odot \Bar{f}_{l-1}\right) \right) \right)\label{eqn-recur} 
\end{equation}
where $\odot$ denotes the Hadamard product, $W_{l}$ is the weight matrix between layers $l-1$ and $l$, and $\Bar{f}^{(s)}_l = \mathcal{S}(\hat{f}^{(s)}_l)$ (i.e., $\hat{\cdot}$ and $\Bar{\cdot}$ denote representations before and after the non-linear activation function $\mathcal{S}$).
To gather the activations produced by each subnet into a single vector, we sum over subnet activations as $\hat{f}_l = \sum_s \hat{f}^{(s)}_l$.
The Hadamard products in \eqref{eqn-recur} mask out neuron activations---both in $\hat{f}^{(s)}_{l}$ and $\Bar{f}_{l-1}$---that are not relevant to subnet $s$.\footnote{In practice, such masking is not actually performed. Rather, we partition the weight matrix such that inactive neurons are never computed.}

Interestingly, if $\Bar{f}_{l-1}$ is an unbiased estimator of the full network output $f^\star_{l-1}$, then $\hat{f}_l$ is an unbiased estimator of $W_l f^\star_{l-1}$.
To show this, we consider the $j$th entry of
\begin{align*}
n^2 \sum_s m^{(s)}_l \odot \left (W_l \left( m^{(s)}_{l-1} \odot \Bar{f}_{l-1}\right) \right),
\end{align*}
for which the expectation can be written as 
\begin{small}
\begin{align} 
\mathbb{E}\Bigg[n^2 \sum_s \sum_i \sum_{i'} W_{l, j,i} m^{(s)}_{l,i'} m^{(s)}_{l-1,i} \Bar{f}_{l-1, i}\Bigg] \nonumber
&= n^2 \sum_s \sum_i \frac{1}{n^2} W_{l, j,i} \mathbb{E}\left[ \Bar{f}_{l-1, i}\right] \nonumber \\ 
&= \sum_i W^l_{l, j,i} f^\star_{l-1, i} \nonumber
\end{align}
\end{small}
which is precisely the $j$th entry in $W_l f^\star_{l-1}$.

\section{Convergence Guarantees for IST} \label{A:ist_conv}
 
We will discuss the convergence behavior of IST in this section; see Appendix \ref{A:proof} for a full proof of these results

Consider minimizing $\ell(w) = \tfrac{1}{n} \sum_{i = 1}^n \ell_i(w)$ as in Equation \ref{eq:loss}.
Our analysis adopts six assumptions, labeled \textsc{Assumption} 1 through \textsc{Assumption} 6.

\begin{assumption}{($L_i$-smoothness)}
Given component $\ell_i$ of $\ell$ function, there exists constant $L_i > 0$ such that for every $w_1, w_2 \in \mathbb{R}^p$ we have that:
\begin{align*}
\|\nabla \ell_i(w_1) - \nabla \ell_i(w_2)\|_2 \leq L_i \cdot \|w_1 - w_2\|_2
\end{align*}
or, equivalently,
\begin{align*}
\ell_i(w_2) \leq \ell_i(w_1) + \langle \nabla \ell_i(w_1), w_2 - w_1 \rangle + \tfrac{L_i}{2} \|w_1 - w_2\|_2^2.
\end{align*}
Further, define $L_{\max} := \max_i L_i$.
\end{assumption}

\begin{assumption}{($Q$-Lipschitz continuity)}
Given $\ell$ function, there exists constant $Q > 0$ such that for every $w_1, w_2 \in \mathbb{R}^p$ we have that:
\begin{align*}
|\ell(w_1) - \ell(w_2)| \leq Q \cdot \|w_1 - w_2\|_2
\end{align*}
or, equivalently,
$\|\nabla \ell(w) \|_2 \leq Q, \quad \forall w \in \mathbb{R}^p$.
\end{assumption}

\begin{assumption}{(Error Bound)}
Let $w^\star$ denote the global optimum of $\ell$.
Then, under the Error Bound assumption, there exists constant $\mu > 0$ such that for every $w \in \mathbb{R}^p$ we have that:
\begin{align*}
\|\nabla \ell(w)\|_2 \geq \mu \|w^\star - w\|_2
\end{align*}
\end{assumption}

Per \cite{karimi2016linear}, Error Bound $\equiv$ Polyak-{\L}ojasiewicz inequality.

\begin{assumption}{(Stochastic gradient variance)}{\label{assumption:1}}
such that
\begin{align*}
\mathbb{E}_{i_t} \left[ \|\nabla \ell_{i_t}(w)\|_2^2 \right] \leq M + M_f \|\nabla \ell(w)\|_2^2,
\end{align*}
where $\ell_{i_t}$ is a randomly selected component from the sum $\tfrac{1}{n} \sum_{i = 1}^n \ell_i(w)$.
\end{assumption}

Note that we make the distinction between the general indexing term $i$ and the randomly selected index per SGD round, $i_t$. 
We follow the problem formulation in \cite{khaled2019gradient} on compressed iterates, where IST performs the following motions:
\begin{itemize}[leftmargin=0.5cm]
\item Given model $w_t$ at iteration $t$, we generate a mask $\mathcal{M}: \mathbb{R}^p \rightarrow \mathbb{R}^p$ such that:
\begin{align*}
\left(\mathcal{M}(w_t)\right)_i = \begin{cases} \tfrac{w_{t, i}}{\xi}, & \text{with probability } \xi, \\ 0, & \text{with probability } 1 - \xi.\end{cases}
\end{align*}
Input and output neurons are always selected in this mask.
\item Given mask $\mathcal{M}(\cdot)$, we generate the subnetwork as in:
\begin{align*}
z_t \equiv \mathcal{M}(w_t) \in \mathbb{R}^p,
\end{align*}
where $z_t$---a \emph{compressed} version of $w_t$---has zeros at positions for deactivated subnetwork weights at iteration $t$ and non-zeros for active weights.
\item We perform gradient descent on the compressed $z_t$:
\begin{align*}
w_{t+1} = z_t - \eta \nabla \ell_{i_t}(z_t),
\end{align*}
with $\eta$ the learning rate and $i_t$ randomly selected from $[n]$.
\end{itemize}
This setting resembles \emph{gradient descent with compressed iterates} (GDCI) \cite{khaled2019gradient}, but our analysis considers a different function class.
Our final two assumptions are on $\mathcal{M}(\cdot)$ with respect to the gradient of $\ell$.

\begin{assumption}{(Additive gradient error assumption with bounded energy)}{\label{ass:norm_condition_0}}
Let $w_t$ be the current model and let $z_t = \mathcal{M}(w_t)$ be the compressed model. 
Consider the stochastic gradient term $\nabla \ell_{i_t}(z_t)$; we assume that, on expectation, the following holds:
\begin{align*}
\mathbb{E}_{\mathcal{M}, i_t}&\left[ \nabla \ell_{i_t}(z_t) ~|~ w_t \right] = \nabla \ell(w_t) + \varepsilon_t, \quad 
\end{align*}
for $ \varepsilon_t \in \mathbb{R}^p$ such that $\|\varepsilon_t\|_2 \leq B$ for $ B > 0$.
\end{assumption}

\begin{assumption}{(Norm condition)}{\label{ass:norm_condition}}
$\exists~\theta \in [0, 1)$ such that: 
\begin{align*}
\left\| \mathbb{E}_{\mathcal{M}, i_t}\left[ \nabla \ell_{i_t}(z_t) ~|~ w_t \right] - \nabla \ell(w_t) \right\|_2 = \|\varepsilon_t\|_2 \leq \theta \|\nabla \ell(w_t)\|_2,
\end{align*}
where $w_t$ and $z_t$ are current and compressed models, respectively.
\end{assumption}

The above assumption is commonly used in derivative free optimization \cite{carter1991global, berahas2019theoretical}.
We are now able to derive the following theorem and corollary, which imply the convergence of IST.

\begin{theorem}{\label{thm:1}}
Let $\ell(w) := \tfrac{1}{n} \sum_{i = 1}^n \ell_i(w)$ have $L_i$-smooth components $\ell_i$ for $L_{\max} := \max_i L_i$, and consider the following recursion:
\begin{align*}
w_{t+1} = z_t - \eta \nabla \ell_{i_t}(z_t), \quad \text{where } z_t = \mathcal{M}(w_t).
\end{align*}

Suppose $\mathcal{M}(w_t)$ and $\ell$ satisfy Assumption \ref{ass:norm_condition_0}.
After $T$ iterations for step size $\eta = \tfrac{1}{2L_{\max}}$, we obtain:
\begin{align*}
&\min_{t \in \{0, \dots, T\}} \mathbb{E}_{\mathcal{M}, i_t}\left[\|\nabla \ell(w_t)\|_2^2\right] \\
\leq& \tfrac{\ell(x_0) - \ell(w^\star)}{\alpha (T+1)} + \tfrac{1}{\alpha} \cdot \left(\tfrac{BQ}{2L_{\max}} + \tfrac{5L_{\max} \omega}{2}\cdot \|w^\star\|_2^2 + \tfrac{M}{4L_{\max}} \right)
\end{align*}
where the expectation is over the random selection on the compression operator $\mathcal{M}(\cdot)$ and the stochastic selection $i_t$, $\alpha = \tfrac{1}{2L_{\max}}\left(1 - \tfrac{M_f}{2}\right) - \tfrac{5\omega L_{\max}}{2\mu^2}$, and $\omega = \tfrac{1 - \xi}{\xi} < \tfrac{\mu^2}{10 L_{\max}^2}$.
\end{theorem}

If we exchange the bounded assumption $\|\varepsilon_t\|_2 \leq B$ and $Q$-Lipschitzness, with the norm condition in Assumption \ref{ass:norm_condition}, we obtain the following corollary.
\begin{corollary}
Let $\ell$ be $L$-smooth, and consider the recursion over compressed iterates:
\begin{align*}
w_{t+1} = z_t - \eta \nabla \ell_{i_t}(z_t), \quad \text{where } z_t = \mathcal{M}(w_t).
\end{align*}
We further assume that the operator mask, along with $\ell$, satisfy the norm condition Assumption \ref{ass:norm_condition} with parameter $\theta \in [0, 1)$. 
Then, after running the above recursion for $T$ iterations for step size $\eta = \tfrac{1}{2L_{\max}}$, we obtain:
\begin{align*}
&\min_{t \in \{0, \dots, T\}} \mathbb{E}_{\mathcal{M}, i_t}\left[\|\nabla \ell(w_t)\|_2^2\right] \\
\leq &\tfrac{\ell(w_0) - \ell(w^\star)}{\alpha (T+1)} + \tfrac{1}{\alpha} \cdot \left(\tfrac{5L_{\max}\omega}{2}\cdot \|w^\star\|_2^2 + \tfrac{M}{4L_{\max}} \right)
\end{align*}
where the expectation is over the random selection on the compression operator $\mathcal{M}(\cdot)$ and $\alpha$ and $\omega$ are expressed as:
\begin{align*}
\alpha &= \tfrac{1}{2L_{\max}}\left(\tfrac{1}{2} - \theta - \tfrac{M_f}{2}\right) - \tfrac{5L_{\max}}{2} \cdot \tfrac{\omega}{\mu^2} \\
\omega &= \tfrac{1 - \xi}{\xi} < \tfrac{\mu^2}{5L_{\max}^2\left(\tfrac{1}{2} -\theta - \tfrac{M_f}{2}\right)}
\end{align*}
\end{corollary}

\section{Proof of Convergence Guarantees for IST} \label{A:proof}
 
We will discuss the convergence behavior of IST in this section.

Consider the problem of minimizing an average of loss functions:
\begin{align*}
x^\star \in \arg\min_{x \in \mathbb{R}^p} \left\{ f(x) := \tfrac{1}{n} \sum_{i = 1}^n f_i(x) \right\}.
\end{align*}
Here, $f_i(\cdot)$ denotes the contribution to the loss of the $i$-th data point.
The components $f_i(\cdot)$ define the nature of the full function $f$: if $f_i$'s are quadratics, we obtain the convex linear regression problem; if $f_i$'s model the forward pass of a neural network, we obtain neural network inference. 

In this note, we will consider $f_i$ functions that do not follow exactly the architecture of a specific neural network, but satisfy general loss assumptions that could potentially be satisfied by neural network models.

\begin{assumption}{($L_i$-smoothness)}
Given component $f_i$ of $f$ function, there exists constant $L_i > 0$ such that for every $x, y \in \mathbb{R}^p$ we have that:
\begin{align*}
\|\nabla f_i(x) - \nabla f_i(y)\|_2 \leq L_i \cdot \|x - y\|_2
\end{align*}
or, equivalently,
\begin{align*}
f_i(y) \leq f_i(x) + \langle \nabla f_i(x), y - x \rangle + \tfrac{L_i}{2} \|x - y\|_2^2.
\end{align*}
Further, define $L_{\max} := \max_i L_i$.
\end{assumption}

Another assumption that we use in part of our results is $Q$-Lipschitz assumption:
\begin{assumption}{($Q$-Lipschitz continuity)}
Given $f$ function, there exists constant $Q > 0$ such that for every $x, y \in \mathbb{R}^p$ we have that:
\begin{align*}
|f(x) - f(y)| \leq Q \cdot \|x - y\|_2
\end{align*}
or, equivalently,
\begin{align*}
\|\nabla f(x) \|_2 \leq Q, \quad \forall x \in \mathbb{R}^p.
\end{align*}
\end{assumption}

Two other assumptions for $f$ functions that do not imply convexity but help as proof convergence rate techniques are:
\begin{assumption}{(Error Bound)}
Let $x^\star$ denote the global optimum of $f$.
Then, under the Error Bound assumption, there exists constant $\mu > 0$ such that for every $x\in \mathbb{R}^p$ we have that:
\begin{align*}
\|\nabla f(x)\|_2 \geq \mu \|x^\star - x\|_2
\end{align*}
\end{assumption}

Per \cite{karimi2016linear}, Error Bound $\equiv$ Polyak-{\L}ojasiewicz inequality.

Regarding stochasticity in gradient descent, we will use the following general assumptions on the boundedness of stochastic gradient variance.
\begin{assumption}{(Stochastic gradient variance)}{\label{A:assumption:1}}
We assume that there are constants $M, M_f > 0$,
such that
\begin{align*}
\mathbb{E}_{i_t} \left[ \|\nabla f_{i_t}(x)\|_2^2 \right] \leq M + M_f \|\nabla f(x)\|_2^2,
\end{align*}
where $f_{i_t}$ is a randomly selected component from the sum $\tfrac{1}{n} \sum_{i = 1}^n f_i(x)$.
\end{assumption}

For the rest of the text, we make the distinction between the general indexing term $i$ and the randomly selected index per SGD round, $i_t$. The differences are clear from the context. 

\subsection{Compressed iterates}
We follow the problem formulation in \cite{khaled2019gradient} on compressed iterates.
Working on IST in a sequentially centralized fashion, IST performs the following motions:
\begin{itemize}
\item Given current model $x_t$ at iteration $t$, we generate a mask $\mathcal{M}: \mathbb{R}^p \rightarrow \mathbb{R}^p$ such that:
\begin{align*}
\left(\mathcal{M}(x_t)\right)_i = \begin{cases} \tfrac{x_{t, i}}{\xi}, & \text{with probability } \xi, \\ 0, & \text{with probability } 1 - \xi.\end{cases}
\end{align*}
This mask deviates from the proposed model in the sense that the input and output neurons in the neural network are always selected. 
\item Given mask $\mathcal{M}(\cdot)$, we generate the subnetwork as in:
\begin{align*}
y_t \equiv \mathcal{M}(x_t) \in \mathbb{R}^p,
\end{align*}
where $y_t$ has zeros at the positions where neurons are deactivated for this subnetwork at iteration $t$ and non-zeros for the active weights. I.e., $y_t$ constitutes a \emph{compressed} version of the full model $x_t$.
\item We perform gradient descent on the compressed $y_t$ as in:
\begin{align*}
x_{t+1} = y_t - \eta \nabla f_{i_t}(y_t),
\end{align*}
for $\eta$ being the learning rate, and $i_t$ being selected randomly from $[n]$.
\end{itemize}
The above setting resembles that of \emph{gradient descent with compressed iterates} (GDCI) in \cite{khaled2019gradient}. 
Our analysis differentiates in that we consider a different function class. 

\textit{Compression operators $\mathcal{M}(\cdot)$}
Let us describe and prove some properties of the compression operator $\mathcal{M}(\cdot)$.

\begin{property}{(Unbiasedness of $\mathcal{M}(\cdot)$)}{\label{prop:1}}
The operator $\mathcal{M}: \mathbb{R}^p \rightarrow \mathbb{R}^p$ as defined above is unbiased; i.e., 
\begin{align*}
\mathbb{E}_{\mathcal{M}}\left[ \mathcal{M}(x) ~|~ x \right] = x, \quad \forall x \in \mathbb{R}^p
\end{align*}
\end{property}

\begin{proof}
To see this, we compute:
\begin{align*}
\mathbb{E}_{\mathcal{M}}\left[ \mathcal{M}(x) ~|~ x \right] = & \begin{bmatrix} \mathbb{E}_{\mathcal{M}}\left[ \left(\mathcal{M}(x)\right)_1 ~|~ x \right] \\ \mathbb{E}_{\mathcal{M}}\left[ \left(\mathcal{M}(x)\right)_2 ~|~ x \right] \\ \vdots \\ \mathbb{E}_{\mathcal{M}}\left[ \left(\mathcal{M}(x)\right)_p ~|~ x \right] \end{bmatrix} \\
=& \begin{bmatrix} \xi \cdot \tfrac{x_1}{\xi} + (1 - \xi) \cdot 0 \\ \xi \cdot \tfrac{x_2}{\xi} + (1 - \xi) \cdot 0 \\ \vdots \\ \xi \cdot \tfrac{x_p}{\xi} + (1 - \xi) \cdot 0 \end{bmatrix} = \begin{bmatrix} x_1 \\ x_2 \\ \vdots \\ x_p \end{bmatrix} = x
\end{align*}
\end{proof}

\begin{property}{(Bounded variance of $\mathcal{M}(\cdot)$)}{\label{prop:2}}
The operator $\mathcal{M}: \mathbb{R}^p \rightarrow \mathbb{R}^p$ has bounded variance as in: 
\begin{align*}
\mathbb{E}_{\mathcal{M}}\left[ \|\mathcal{M}(x) - x \|_2^2 ~|~ x \right] = \tfrac{1 - \xi}{\xi} \|x\|_2^2, \quad \forall x \in \mathbb{R}^p
\end{align*}
\end{property}

\begin{proof}
Let us first expand the squared term:
\begin{align*}
&\mathbb{E}_{\mathcal{M}}\left[ \|\mathcal{M}(x) - x \|_2^2 ~|~ x \right] \\
= &\mathbb{E}_{\mathcal{M}}\left[ \|\mathcal{M}(x)\|_2^2 + \|x\|_2^2 - 2\langle \mathcal{M}(x), x \rangle ~|~ x \right] \\ 
= &\mathbb{E}_{\mathcal{M}}\left[ \|\mathcal{M}(x)\|_2^2 ~|~ x \right] + \|x\|_2^2 - 2\langle \mathbb{E}_{\mathcal{M}}\left[\mathcal{M}(x) ~|~ x \right], x \rangle \\ 
= &\mathbb{E}_{\mathcal{M}}\left[ \|\mathcal{M}(x)\|_2^2 ~|~ x \right] - \|x\|_2^2
\end{align*}
Focusing on the first term on the right hand side:
\begin{align*}
&\mathbb{E}_{\mathcal{M}}\left[ \|\mathcal{M}(x)\|_2^2 ~|~ x \right] \\
=&\mathbb{E}_{\mathcal{M}}\left[ \sum_{i = 1}^p \left(\mathcal{M}(x)\right)_i^2 ~|~ x \right] 
= \sum_{i = 1}^p \mathbb{E}_{\mathcal{M}}\left[ \left(\mathcal{M}(x)\right)_i^2 ~|~ x \right] \\
=& \sum_{i = 1}^p \left(\xi \cdot \tfrac{x_i^2}{\xi^2} + (1- \xi) \cdot 0 \right) = \tfrac{1}{\xi} \sum_{i = 1}^p x_i^2 = \tfrac{1}{\xi} \|x\|_2^2
\end{align*}
Combining the above we get:
\begin{align*}
\mathbb{E}_{\mathcal{M}}\left[ \|\mathcal{M}(x) - x \|_2^2 ~|~ x \right] =  \tfrac{1}{\xi} \|x\|_2^2 - \|x\|_2^2 = \tfrac{1 - \xi}{\xi} \cdot \|x\|_2^2
\end{align*}
\end{proof}

Some assumptions we make for $\mathcal{M}(\cdot)$ with regard to the gradient of $f$ are as follows.
\begin{assumption}{(Additive gradient error assumption with bounded energy)}
Let $x_t$ be the current model and let $y_t = \mathcal{M}(x_t)$ be the compressed model. 
Consider the stochastic gradient term $\nabla f_{i_t}(y_t)$; we assume that, on expectation, the following additive noise assumption holds:
\begin{align*}
\mathbb{E}_{\mathcal{M}, i_t}\left[ \nabla f_{i_t}(y_t) ~|~ x_t \right] = \nabla f(x_t) + \varepsilon_t, \quad \\\text{for } \varepsilon_t \in \mathbb{R}^p \text{ such that } \|\varepsilon_t\|_2 \leq B \text{ for } B > 0.
\end{align*}
\end{assumption}

A different assumption that can be made for stochastic gradient $\nabla f_{i_t}(y_t)$ is the \emph{norm condition}, used in derivative free optimization \cite{carter1991global, berahas2019theoretical}:
\begin{assumption}{(Norm condition)}{\label{A:ass:norm_condition}}
There is a constant $\theta \in [0, 1)$ such that: 
\begin{align*}
\left\| \mathbb{E}_{\mathcal{M}, i_t}\left[ \nabla f_{i_t}(y_t) ~|~ x_t \right] - \nabla f(x_t) \right\|_2 = \|\varepsilon_t\|_2 \leq \theta \|\nabla f(x_t)\|_2,
\end{align*}
where $x_t$ is the current model and $y_t = \mathcal{M}(x_t)$ is the compressed model.
\end{assumption}

\subsection{Proof of sequential IST}
In this subsection, we will provide the backbone of our proof; later on we will make different assumptions that will lead to different final results. 
For this first part, we will only use the basic properties of the compression operator $\mathcal{M}(\cdot)$ and the $L$-smoothness of $f$.

We start with the following Lemma.
\begin{lemma}{\label{lemma:2}}
Let $y_t = \mathcal{M}(x_t)$. Then, for $x^\star$ the optimal point for $f$, it holds:
\begin{align*}
\mathbb{E}_{\mathcal{M}}\left[ \| y_t - x_t \|_2^2 ~|~ x_t\right] \leq \tfrac{2(1 - \xi)}{\xi} \|x_t - x^\star\|_2^2 + \tfrac{2(1-\xi)}{\xi} \|x^\star\|_2^2.
\end{align*}
\end{lemma}

\begin{proof}
By Property \ref{prop:2}, we know that:
\begin{align*}
\mathbb{E}_{\mathcal{M}}\left[ \|y_t - x_t \|_2^2 ~|~ x_t \right] = \tfrac{1 - \xi}{\xi} \|x_t\|_2^2.
\end{align*}
Combining this with the property that:
\begin{align*}
\|x_t\|_2^2 = \|x_t - x^\star + x^\star\|_2^2 \leq 2\|x_t - x^\star\|_2^2 + 2 \|x^\star\|_2^2,
\end{align*}
we get the reported result.
\end{proof}

Moreover, we also state the following Lemma, which is used in most proofs below.
\begin{lemma}{\label{lemma:3}}
Let $y_t = \mathcal{M}(x_t)$. Then,
\begin{align*}
&\mathbb{E}_{\mathcal{M}, i_t}\left[ \| y_t - x_t - \eta \nabla f_{i_t}(y_t) \|_2^2 ~|~ x_t\right] \\ \leq &\tfrac{4(1 + \eta^2 L_{\max}^2)(1-\xi)}{\xi} \left(\|x_t - x^\star\|_2^2 + \|x^\star\|_2^2 \right) + 2\eta^2 \mathbb{E}_{i_t}\left[\|\nabla f_{i_t}(x_t)\|_2^2 ~|~ x_t\right].
\end{align*}
\end{lemma}

\begin{proof}
For any $z \in \mathbb{R}^p$, we have:
\begin{align*}
&\mathbb{E}_{\mathcal{M}, i_t} \left[ \| y_t - x_t - \eta \nabla f_{i_t}(y_t) \|_2^2 ~|~ x_t\right] \\
=& \mathbb{E}_{\mathcal{M}, i_t} \left[ \| y_t - x_t -\eta \nabla f_{i_t}(z) + \eta \nabla f_{i_t}(z) - \eta \nabla f_{i_t}(y_t) \|_2^2 ~|~ x_t\right] \\
\leq & 2 \cdot \mathbb{E}_{\mathcal{M}, i_t}\left[ \| y_t - x_t -\eta \nabla f_{i_t}(z)\|_2^2 ~|~ x_t\right] \\
&\: + 2\eta^2 \cdot  \mathbb{E}_{\mathcal{M}, i_t}\left[ \|\nabla f_{i_t}(z) - \nabla f_{i_t}(y_t)\|_2^2 ~|~ x_t\right] \\ 
=& 2 \cdot \mathbb{E}_{\mathcal{M}, i_t}\left[ \| y_t - x_t\|_2^2 + \eta^2 \|\nabla f_{i_t}(z)\|_2^2  - 2\eta \cdot \langle \nabla f_{i_t}(z), y_t - x_t \rangle ~|~ x_t\right] \\
&\: + 2\eta^2 \cdot  \mathbb{E}_{\mathcal{M}, i_t}\left[ \|\nabla f_{i_t}(z) - \nabla f_{i_t}(y_t)\|_2^2 ~|~ x_t\right]
\end{align*}
Observe that 
\begin{align*}
\mathbb{E}_{\mathcal{M}, i_t}\left[ \langle \nabla f_{i_t}(z), y_t - x_t \rangle ~|~ x_t\right] &= \langle\mathbb{E}_{i_t}\left[ \nabla f_{i_t}(z)\right], \mathbb{E}_{\mathcal{M}}\left[ y_t ~|~ x_t\right] - x_t \rangle \\ 
&= \langle \nabla f(z), x_t - x_t \rangle = 0
\end{align*} 
due to Property \ref{prop:1}.
Then:
\begin{align*}
&\mathbb{E}_{\mathcal{M}, i_t}\left[ \| y_t - x_t - \eta \nabla f(y_t) \|_2^2 ~|~ x_t\right] \\
\leq & 2 \cdot \mathbb{E}_{\mathcal{M}}\left[ \| y_t - x_t\|_2^2 ~|~ x_t\right] + 2 \eta^2 \mathbb{E}_{i_t}\left[\|\nabla f_{i_t}(z)\|_2^2 ~|~ x_t \right] \\ 
& + 2\eta^2 \cdot  \mathbb{E}_{\mathcal{M}, i_t}\left[ \|\nabla f_{i_t}(z) - \nabla f_{i_t}(y_t)\|_2^2 ~|~ x_t\right]
\end{align*}

By $L_i$-smoothness, we have:
\begin{align*}
& \mathbb{E}_{\mathcal{M}, i_t}\left[ \|\nabla f(z) - \nabla f(y_t)\|_2^2 ~|~ x_t\right] \\
\leq & \mathbb{E}_{\mathcal{M}, i_t}\left[ L_{i_t}^2 \cdot \|z - y_t\|_2^2 ~|~ x_t\right] \\
\leq & L_{\max}^2 \cdot \mathbb{E}_{\mathcal{M}, i_t}\left[ \|z - y_t\|_2^2 ~|~ x_t\right] \\
\leq & L_{\max}^2 \cdot \mathbb{E}_{\mathcal{M}}\left[\|z - x_t + x_t - y_t\|_2^2 ~|~ x_t\right] \\
= & L_{\max}^2 \cdot \mathbb{E}_{\mathcal{M}}\left[\|z - x_t\|_2^2 +  \|x_t - y_t\|_2^2 + 2\langle z - x_t, x_t - y_t \rangle ~|~ x_t\right] \\
= & L_{\max}^2 \cdot \mathbb{E}_{\mathcal{M}}\left[\|z - x_t\|_2^2~|~ x_t\right] + L_{\max}^2 \cdot \mathbb{E}_{\mathcal{M}}\left[\|x_t - y_t\|_2^2 ~|~ x_t\right]
\end{align*}
In the above, we used Property \ref{prop:1} to get $\mathbb{E}_{\mathcal{M}}\left[y_t ~|~ x_t\right] = x_t$.
Using this result in the expression above, we get:
\begin{align*}
&\mathbb{E}_{\mathcal{M}, i_t} \left[ \| y_t - x_t - \eta \nabla f_{i_t}(y_t) \|_2^2 ~|~ x_t\right] \\ 
\leq& 2 \cdot \mathbb{E}_{\mathcal{M}}\left[ \| y_t - x_t\|_2^2 ~|~ x_t\right] + 2\eta^2 \mathbb{E}_{i_t}\left[\|\nabla f_{i_t}(z)\|_2^2 ~|~ x_t\right] \\ 
&\: + 2\eta^2 \cdot L_{\max}^2 \cdot \left( \mathbb{E}_{\mathcal{M}}\left[\|z - x_t\|_2^2~|~ x_t\right] + \mathbb{E}_{\mathcal{M}}\left[\|x_t - y_t\|_2^2 ~|~ x_t\right]\right) \\
=& 2 (1 + \eta^2 L_{\max}^2) \cdot \mathbb{E}_{\mathcal{M}}\left[ \| y_t - x_t\|_2^2 ~|~ x_t\right] + 2\eta^2 \mathbb{E}_{i_t}\left[\|\nabla f_{i_t}(z)\|_2^2 ~|~ x_t\right] 
\\ &\: + 2\eta^2 \cdot L_{\max}^2 \cdot \mathbb{E}_{\mathcal{M}}\left[\|z - x_t\|_2^2~|~ x_t\right]
\end{align*}
The above expression holds for any $z \in \mathbb{R}^p$, and thus it will hold for $z = x_t$. 
In this case, the above inequality becomes:
\begin{align*}
& \mathbb{E}_{\mathcal{M}, i_t}\left[ \| y_t - x_t - \eta \nabla f_{i_t}(y_t) \|_2^2 ~|~ x_t\right] \\ \leq &2 (1 + \eta^2 L_{\max}^2) \cdot \mathbb{E}_\mathcal{M}\left[ \| y_t - x_t\|_2^2 ~|~ x_t\right] + 2\eta^2 \mathbb{E}_{i_t}\left[\|\nabla f_{i_t}(x_t)\|_2^2 ~|~ x_t\right]
\end{align*}
Finally, using Lemma \ref{lemma:2}, we obtain the desired inequality.
\end{proof}

The following lemma leads to a general recursion over function values, that will lead to specific convergence rate guarantees later on, based on additional assumptions.
\begin{lemma}{\label{lemma:1}}
Let $f$ be $L$-smooth, and consider the recursion over compressed iterates:
\begin{align*}
x_{t+1} = y_t - \eta \nabla f_{i_t}(y_t), \quad \text{where } y_t = \mathcal{M}(x_t).
\end{align*}
Then the following recursion holds:
\begin{align*}
\mathbb{E}_{\mathcal{M}, i_t}\left[ f(x_{t+1}) ~|~ x_t \right] \leq & f(x_t) - \eta \cdot \langle \nabla f(x_t), \mathbb{E}_{\mathcal{M}, i_t}\left[\nabla f_{i_t}(y_t)  ~|~ x_t \right] \rangle \\& + \tfrac{L_i}{2} \cdot \mathbb{E}_{\mathcal{M}, i_t}\left[ \|y_t - \eta \nabla f_{i_t}(y_t) - x_t\|_2^2 ~|~ x_t \right],
\end{align*}
where the expectation is over the random selection on the compression operator $\mathcal{M}(\cdot)$ and the stochasticity of the gradient.
\end{lemma}

\begin{proof}
Starting from the $L_i$-smoothness condition, we have on expectation with respect to $\mathcal{M}, i_t$ at time $t$:
\begin{align*}
&\mathbb{E}_{\mathcal{M}, i_t} \left[ f(x_{t+1}) ~|~ x_t \right] \\ \leq & \mathbb{E}_{\mathcal{M}, i_t}\left[f(x_t) + \langle \nabla f_{i_t}(x_t), x_{t+1} - x_t \rangle + \tfrac{L_i}{2} \|x_{t+1} - x_t\|_2^2 ~|~ x_t \right] \\
=& \mathbb{E}_{\mathcal{M}, i_t}\left[f(x_t) ~|~ x_t \right] +  \mathbb{E}_{\mathcal{M}, i_t}\left[\langle \nabla f_{i_t}(x_t), x_{t+1} - x_t \rangle ~|~ x_t \right] \\
& \: + \tfrac{L_i}{2} \cdot \mathbb{E}_{\mathcal{M}, i_t}\left[ \|x_{t+1} - x_t\|_2^2 ~|~ x_t \right] \\
=& f(x_t) +  \mathbb{E}_{\mathcal{M}, i_t}\left[\langle \nabla f_{i_t}(x_t), y_t - \eta \nabla f_{i_t}(y_t) - x_t \rangle ~|~ x_t \right] \\
& \: + \tfrac{L_i}{2} \cdot \mathbb{E}_{\mathcal{M}, i_t}\left[ \|y_t - \eta \nabla f_{i_t}(y_t) - x_t\|_2^2 ~|~ x_t \right] \\
= & f(x_t) +  \mathbb{E}_{\mathcal{M}, i_t}\left[\langle \nabla f_{i_t}(x_t), y_t - x_t \rangle ~|~ x_t \right] \\
& \: - \eta \cdot \mathbb{E}_{\mathcal{M}, i_t}\left[\langle \nabla f_{i_t}(x_t), \nabla f_{i_t}(y_t) \rangle ~|~ x_t \right] \\ 
& \: + \tfrac{L_i}{2} \cdot \mathbb{E}_{\mathcal{M}, i_t}\left[ \|y_t - \eta \nabla f(y_t) - x_t\|_2^2 ~|~ x_t \right] \\
= & f(x_t) +  \langle  \mathbb{E}_{i_t} [\nabla f_{i_t} (x_t) ~|~ x_t], \mathbb{E}_{\mathcal{M}}\left[y_t~|~ x_t \right] - x_t \rangle  \\
& \: - \eta \cdot \langle \mathbb{E}_{i_t} [\nabla f_{i_t} (x_t) ~|~ x_t], \mathbb{E}_{\mathcal{M}, i_t}\left[\nabla f_{i_t}(y_t)  ~|~ x_t \right] \rangle \\ 
& \: + \tfrac{L_i}{2} \cdot \mathbb{E}_{\mathcal{M}, i_t}\left[ \|y_t - \eta \nabla f_{i_t}(y_t) - x_t\|_2^2 ~|~ x_t \right] \\
&\!\!\!\!\!\!\!\!\stackrel{\text{Property \ref{prop:1}}}{=} f(x_t) - \eta \cdot \langle \nabla f(x_t), \mathbb{E}_{\mathcal{M}, i_t}\left[\nabla f_{i_t}(y_t)  ~|~ x_t \right] \rangle \\
& \quad \quad + \tfrac{L_i}{2} \cdot \mathbb{E}_{\mathcal{M}, i_t}\left[ \|y_t - \eta \nabla f_{i_t}(y_t) - x_t\|_2^2 ~|~ x_t \right]
\end{align*}
In the last step, we also used the unbiasedness of the stochastic gradient.
\end{proof}

Our proof will branch out for different assumptions we make. 
We begin with the following ``cocktail'' of assumptions. 
\begin{theorem}{\label{A:thm:1}}
Let $f := \tfrac{1}{n} \sum_{i = 1}^n f_i(x)$ has $L_i$-smooth components $f_i$ for $L_{\max} := \max_i L_i$, and consider the recursion over compressed iterates:
\begin{align*}
x_{t+1} = y_t - \eta \nabla f_{i_t}(y_t), \quad \text{where } y_t = \mathcal{M}(x_t).
\end{align*}
In additional to Lemma \ref{lemma:1}, we further assume that $f$ is $Q$-Lipschitz, and satisfies the Error Bound with parameter $\mu > 0$.
Further, our recursion should satisfy the gradient boundedness Assumption \ref{A:assumption:1} for $M > 0$ and $0 < M_f < 1$.
Finally, we assume that the operator mask, along with $f$, satisfy the additive gradient error assumption with bounded energy such that:
\begin{align*}
\mathbb{E}_{\mathcal{M}, i_t}\left[ \nabla f_{i_t}(y_t) ~|~ x_t \right] = \nabla f(x_t) + \varepsilon_t, \quad \\
\text{for } \varepsilon_t \in \mathbb{R}^p \text{ such that } \|\varepsilon_t\|_2 \leq B \text{ for } B > 0.
\end{align*}
Then, after running the above recursion for $T$ iterations for step size $\eta = \tfrac{1}{2L_{\max}}$, we obtain:
\begin{align*}
&\min_{t \in \{0, \dots, T\}} \mathbb{E}_{\mathcal{M}, i_t}\left[\|\nabla f(x_t)\|_2^2\right] \\
\leq &\tfrac{f(x_0) - f(x^\star)}{\alpha (T+1)} +\tfrac{1}{\alpha} \cdot \left(\tfrac{BQ}{2L_{\max}} + \tfrac{5L_{\max} \omega}{2}\cdot \|x^\star\|_2^2 + \tfrac{M}{4L_{\max}} \right)
\end{align*}
where the expectation is over the random selection on the compression operator $\mathcal{M}(\cdot)$ and the stochastic selection $i_t$, $\alpha = \tfrac{1}{2L_{\max}}\left(1 - \tfrac{M_f}{2}\right) - \tfrac{5\omega L_{\max}}{2\mu^2}$, and $\omega = \tfrac{1 - \xi}{\xi} < \tfrac{\mu^2}{10 L_{\max}^2}$.
\end{theorem}

\begin{proof}
We know from Lemma \ref{lemma:1} that the following holds:
\begin{align*}
\mathbb{E}_{\mathcal{M}, i_t}\left[ f(x_{t+1}) ~|~ x_t \right] \leq & f(x_t) - \eta \cdot \langle \nabla f(x_t), \mathbb{E}_{\mathcal{M}, i_t}\left[\nabla f_{i_t}(y_t)  ~|~ x_t \right] \rangle \\
&+ \tfrac{L_i}{2} \cdot \mathbb{E}_{\mathcal{M}, i_t}\left[ \|y_t - \eta \nabla f_{i_t}(y_t) - x_t\|_2^2 ~|~ x_t \right],
\end{align*}
Using the additive gradient noise assumption $\mathbb{E}_{\mathcal{M}, i_t}\left[ \nabla f_{i_t}(y_t) ~|~ x_t \right] = \nabla f(x_t) + \varepsilon_t$, we have:
\begin{align*}
&\mathbb{E}_{\mathcal{M}, i_t}\left[ f(x_{t+1}) ~|~ x_t \right] \\
\leq & f(x_t) - \eta \cdot \langle \nabla f(x_t), \nabla f(x_t) \\
& \: + \varepsilon_t \rangle + \tfrac{L_i}{2} \cdot \mathbb{E}_{\mathcal{M}, i_t}\left[ \|y_t - \eta \nabla f_{i_t}(y_t) - x_t\|_2^2 ~|~ x_t \right] \\
=& f(x_t) - \eta \cdot \| \nabla f(x_t)\|_2^2 - \eta \langle \nabla f(x_t), \varepsilon_t \rangle \\
& \: +  \tfrac{L_i}{2} \cdot \mathbb{E}_{\mathcal{M}, i_t}\left[ \|y_t - \eta \nabla f_{i_t}(y_t) - x_t\|_2^2 ~|~ x_t \right] \\
&\!\!\!\!\!\!\!\!\stackrel{\text{Lemma \ref{lemma:3}}}{\leq} f(x_t) - \eta \cdot \| \nabla f(x_t)\|_2^2 - \eta \langle \nabla f(x_t), \varepsilon_t \rangle \\ 
+& \tfrac{L_i}{2} \cdot \left( \tfrac{4(1 + \eta^2 L_{\max}^2)(1-\xi)}{\xi} \left(\|x_t - x^\star\|_2^2 + \|x^\star\|_2^2 \right) + 2\eta^2 \mathbb{E}_{i_t}\left[\|\nabla f_{i_t}(x_t)\|_2^2 ~|~ x_t\right]\right) \\
&\!\!\!\!\!\!\!\!\stackrel{L_i \leq L_{\max}}{\leq} f(x_t) - \eta \cdot \| \nabla f(x_t)\|_2^2 - \eta \langle \nabla f(x_t), \varepsilon_t \rangle \\ 
+& \tfrac{L_{\max}}{2} \cdot \left( \tfrac{4(1 + \eta^2 L_{\max}^2)(1-\xi)}{\xi} \left(\|x_t - x^\star\|_2^2 + \|x^\star\|_2^2 \right) + 2\eta^2 \mathbb{E}_{i_t}\left[\|\nabla f_{i_t}(x_t)\|_2^2 ~|~ x_t\right]\right) \\
&\!\!\!\!\!\!\!\!\!\!\!\!\stackrel{\text{Assumption \ref{A:assumption:1}}}{\leq} f(x_t) - \eta \cdot \| \nabla f(x_t)\|_2^2 - \eta \langle \nabla f(x_t), \varepsilon_t \rangle \\ 
+& \tfrac{L_{\max}}{2} \cdot \left( \tfrac{4(1 + \eta^2 L_{\max}^2)(1-\xi)}{\xi} \left(\|x_t - x^\star\|_2^2 + \|x^\star\|_2^2 \right) + 2\eta^2 \left(M + M_f \|\nabla f(x_t)\|_2^2 \right)\right) \\
= & f(x_t) - \eta (1 - \eta L_{\max} M_f) \cdot \| \nabla f(x_t)\|_2^2 - \eta \langle \nabla f(x_t), \varepsilon_t \rangle \\ 
&\: + \tfrac{2L_{\max}(1 + \eta^2 L_{\max}^2)(1-\xi)}{\xi} \left(\|x_t - x^\star\|_2^2 + \|x^\star\|_2^2 \right) + L_{\max} \eta^2 M \\
&\!\!\!\!\!\!\!\!\!\!\!\!\stackrel{\text{Error Bound}}{\leq} f(x_t) - \eta (1 - \eta L_{\max} M_f) \cdot \| \nabla f(x_t)\|_2^2 - \eta \langle \nabla f(x_t), \varepsilon_t \rangle \\ 
&\quad \quad + \tfrac{2L_{\max}(1 + \eta^2 L_{\max}^2)(1-\xi)}{\xi} \cdot \tfrac{1}{\mu^2} \cdot \|\nabla f(x_t)\|_2^2 \\
&\quad \quad + \tfrac{2L_{\max}(1 + \eta^2 L_{\max}^2)(1-\xi)}{\xi} \|x^\star\|_2^2 + L_{\max} \eta^2 M
\end{align*}
To simplify notation, define $\omega = \tfrac{1-\xi}{\xi}$. 
Rearranging the terms in the inequality above we obtain:
\begin{align*}
&\mathbb{E}_{\mathcal{M}, i_t}\left[ f(x_{t+1}) ~|~ x_t \right] \\
\leq& f(x_t) - \left(\eta (1 - \eta L_{\max} M_f) - 2L_{\max}(1 + \eta^2 L_{\max}^2) \cdot \tfrac{\omega}{\mu^2}  \right) \cdot \| \nabla f(x_t)\|_2^2 \\ & \: - \eta \langle \nabla f(x_t), \varepsilon_t \rangle + 2\omega L_{\max}(1 + \eta^2 L_{\max}^2) \|x^\star\|_2^2 + L_{\max} \eta^2 M
\end{align*}
Define $g(\eta) = \eta (1 - \eta L_{\max} M_f) - 2L_{\max}(1 + \eta^2 L_{\max}^2) \cdot \tfrac{\omega}{\mu^2}$ which is a quadratic function.
We set the step size as in $\eta = \tfrac{1}{2L_{\max}}$, which is a reasonable assumption based on convex optimization criteria. 
Then, 
\begin{align*}
g\left(\tfrac{1}{2L_{\max}}\right) = \tfrac{1}{2L_{\max}}\left(1- \tfrac{M_f}{2} \right) - \tfrac{5L_{\max}\omega}{2\mu^2} \equiv \alpha.
\end{align*}
For the proof, we require $\alpha > 0$; according to our assumptions, $1 - \tfrac{M_f}{2} > \tfrac{1}{2} \Rightarrow M_f < 1$. 
This further leads to the requirement that $\tfrac{1-\xi}{\xi} < \tfrac{\mu^2}{10L_{\max}^2}$.
The above result into:
\begin{align*}
&\mathbb{E}_{\mathcal{M}, i_t}\left[ f(x_{t+1}) ~|~ x_t \right] \\
\leq & f(x_t) - \alpha \cdot \| \nabla f(x_t)\|_2^2 - \eta \langle \nabla f(x_t), \varepsilon_t \rangle \\ 
&\: + 2\omega L_{\max}(1 + \eta^2 L_{\max}^2) \|x^\star\|_2^2 + L_{\max} \eta^2 M \\ 
&\!\!\!\!\!\!\!\!\!\!\!\!\!\!\stackrel{\text{Cauchy-Schwarz}}{\leq} f(x_t) - \alpha \cdot \| \nabla f(x_t)\|_2^2 + \eta \cdot \|\nabla f(x_t)\|_2 \cdot\|\varepsilon_t\|_2 \\ 
&\: + 2\omega L_{\max}(1 + \eta^2 L_{\max}^2) \|x^\star\|_2^2 + L_{\max} \eta^2 M \\
&\!\!\!\!\!\!\!\!\!\!\stackrel{\text{$\eta = \tfrac{1}{2L_{\max}}$}}{=} f(x_t) - \alpha \cdot \| \nabla f(x_t)\|_2^2 + \tfrac{1}{2L_{\max}} \cdot \|\nabla f(x_t)\|_2 \cdot\|\varepsilon_t\|_2 \\ 
&\: + \tfrac{5L_{\max}\omega}{2}\cdot \|x^\star\|_2^2 + \tfrac{M}{4L_{\max}}\\
&\!\!\!\!\!\!\!\!\!\!\stackrel{\text{$Q$-Lipschitz}}{\leq} f(x_t) - \alpha \cdot \| \nabla f(x_t)\|_2^2 + \tfrac{Q}{2L_{\max}} \cdot\|\varepsilon_t\|_2 \\
&\: + \tfrac{5L_{\max}\omega}{2}\cdot \|x^\star\|_2^2 + \tfrac{M}{4L_{\max}} \\
&\!\!\!\!\!\!\!\!\stackrel{\|\varepsilon_t\|_2 \leq B}{\leq} f(x_t) - \alpha \cdot \| \nabla f(x_t)\|_2^2 + \tfrac{BQ}{2L_{\max}} + \tfrac{5L_{\max}\omega}{2}\cdot \|x^\star\|_2^2 + \tfrac{M}{4L_{\max}}
\end{align*}
Using the law of total expectation $\mathbb{E}[X] = \mathbb{E}[\mathbb{E}[X ~|~Y]]$, we have:
\begin{align*}
\mathbb{E}\left[ f(x_{t+1})\right] = \mathbb{E}\left[ \mathbb{E}\left[f(x_{t+1}) ~|~ x_t \right]\right]
\end{align*}
and thus:
\begin{align*}
\mathbb{E}_{\mathcal{M}, i_t}\left[ f(x_{t+1}) \right] \leq &\mathbb{E}_{\mathcal{M}, i_t}\left[ f(x_t)\right] - \alpha \cdot \mathbb{E}_{\mathcal{M}, i_t}\left[ \| \nabla f(x_t)\|_2^2\right] \\ 
&+ \tfrac{BQ}{2L} + \tfrac{5L_{\max}\omega}{2}\cdot \|x^\star\|_2^2  + \tfrac{M}{4L_{\max}}
\end{align*}
Using the fact that $f(x^\star) \leq \mathbb{E}_{\mathcal{M}, i_t}\left[ f(x_{T+1}) \right]$, and telescoping over $T$ iterations, we obtain:
\begin{align*}
&f(x^\star) \leq \mathbb{E}_{\mathcal{M}, i_t}\left[ f(x_{T+1}) \right] \\
\leq & f(x_0) - \alpha \sum_{t = 0}^T \mathbb{E}_{\mathcal{M}, i_t}\left[\|\nabla f(x_t)\|_2^2\right] \\
&+ (T+1) \cdot \left(\tfrac{BQ}{2L} + \tfrac{5L_{\max}\omega}{2}\cdot \|x^\star\|_2^2  + \tfrac{M}{4L_{\max}} \right)
\end{align*}
which further leads to:
\begin{align*}
&\sum_{t = 0}^T \mathbb{E}_{\mathcal{M}, i_t}\left[\|\nabla f(x_t)\|_2^2\right] \\
\leq &\tfrac{f(x_0) - f(x^\star)}{\alpha} + \tfrac{T+1}{\alpha} \cdot \left(\tfrac{BQ}{2L} + \tfrac{5L_{\max}\omega}{2}\cdot \|x^\star\|_2^2 + \tfrac{M}{4L_{\max}} \right)
\end{align*}
Observe that: $(T+1) \cdot \min_{t \in \{0, \dots, T\}} \mathbb{E}_{\mathcal{M}, i_t}\left[\|\nabla f(x_t)\|_2^2\right]$\\ $\leq \sum_{t = 0}^T \mathbb{E}_{\mathcal{M}, i_t}\left[\|\nabla f(x_t)\|_2^2\right]$, which leads to the final result:
\begin{align*}
& \min_{t \in \{0, \dots, T\}} \mathbb{E}_{\mathcal{M}, i_t}\left[\|\nabla f(x_t)\|_2^2\right] \\
\leq & \tfrac{f(x_0) - f(x^\star)}{\alpha (T+1)} + \tfrac{1}{\alpha} \cdot \left(\tfrac{BQ}{2L} + \tfrac{5L_{\max}\omega}{2}\cdot \|x^\star\|_2^2 + \tfrac{M}{4L_{\max}} \right)
\end{align*}
\end{proof}

In the following corollary, we exchange the bounded assumption $\|\varepsilon_t\|_2 \leq B$ and $Q$-Lipschitzness, with the norm condition in Assumption \ref{A:ass:norm_condition}.
\begin{corollary}
Let $f$ be $L$-smooth, and consider the recursion over compressed iterates:
\begin{align*}
x_{t+1} = y_t - \eta \nabla f_{i_t}(y_t), \quad \text{where } y_t = \mathcal{M}(x_t).
\end{align*}
In additional to Lemma \ref{lemma:1}, we further assume that the operator mask, along with $f$, satisfy the norm condition Assumption \ref{A:ass:norm_condition} with parameter $\theta \in [0, 1)$. 
Then, after running the above recursion for $T$ iterations for step size $\eta = \tfrac{1}{2L_{\max}}$, we obtain:
\begin{align*}
\min_{t \in \{0, \dots, T\}} \mathbb{E}_{\mathcal{M}, i_t}\left[\|\nabla f(x_t)\|_2^2\right] \leq \tfrac{f(x_0) - f(x^\star)}{\alpha (T+1)} + \tfrac{1}{\alpha} \cdot \left(\tfrac{5L_{\max}\omega}{2}\cdot \|x^\star\|_2^2 + \tfrac{M}{4L_{\max}} \right)
\end{align*}
where the expectation is over the random selection on the compression operator $\mathcal{M}(\cdot)$, $\alpha = \tfrac{1}{2L_{\max}}\left(\tfrac{1}{2} - \theta - \tfrac{M_f}{2}\right) - \tfrac{5L_{\max}}{2} \cdot \tfrac{\omega}{\mu^2}$, and $\omega = \tfrac{1 - \xi}{\xi} < \tfrac{\mu^2}{5L_{\max}^2\left(\tfrac{1}{2} -\theta - \tfrac{M_f}{2}\right)}$.
\end{corollary}

\begin{proof}
From Theorem \ref{A:thm:1}, we have:
\begin{align*}
& \mathbb{E}_{\mathcal{M}, i_t}\left[ f(x_{t+1}) ~|~ x_t \right] \\
\leq & f(x_t) - \left(\eta (1 - \eta L_{\max} M_f) - 2L_{\max}(1 + \eta^2 L_{\max}^2) \cdot \tfrac{\omega}{\mu^2}  \right) \cdot \| \nabla f(x_t)\|_2^2 \\ 
&\: - \eta \langle \nabla f(x_t), \mathbb{E}_{\mathcal{M}, i_t}[\varepsilon_t ~|~ x_t] \rangle + 2\omega L_{\max}(1 + \eta^2 L_{\max}^2) \|x^\star\|_2^2 + L_{\max} \eta^2 M \\
&\!\!\!\!\!\!\!\!\!\!\!\!\!\!\stackrel{\text{Cauchy-Schwarz}}{\leq} f(x_t) - \left(\eta (1 - \eta L_{\max} M_f) - 2L_{\max}(1 + \eta^2 L_{\max}^2) \cdot \tfrac{\omega}{\mu^2}  \right) \cdot \| \nabla f(x_t)\|_2^2 \\
&\: + \eta \|\nabla f(x_t)\|_2 \cdot \|\mathbb{E}_{\mathcal{M}, i_t}[\varepsilon_t ~|~ x_t]\|_2 + 2\omega L_{\max}(1 + \eta^2 L_{\max}^2) \|x^\star\|_2^2 + L_{\max} \eta^2 M  \\
=& f(x_t) - \left(\eta (1 - \eta L_{\max} M_f) - 2L_{\max}(1 + \eta^2 L_{\max}^2) \cdot \tfrac{\omega}{\mu^2}  \right) \cdot \| \nabla f(x_t)\|_2^2 \\ 
&\: + \eta \|\nabla f(x_t)\|_2 \cdot \|\mathbb{E}_{\mathcal{M}, i_t}\left[ \nabla f_{i_t}(y_t) ~|~ x_t \right] - \nabla f(x_t)\|_2 \\ 
&\: + 2\omega L_{\max}(1 + \eta^2 L_{\max}^2) \|x^\star\|_2^2 + L_{\max} \eta^2 M  \\
&\!\!\!\!\!\!\!\!\!\!\!\!\!\!\stackrel{\text{Norm condition}}{\leq} f(x_t) - \left(\eta (1 - \eta L_{\max} M_f) - 2L_{\max}(1 + \eta^2 L_{\max}^2) \cdot \tfrac{\omega}{\mu^2}  \right) \cdot \| \nabla f(x_t)\|_2^2 \\ 
&\: + \eta \theta \|\nabla f(x_t)\|_2^2 + 2\omega L_{\max}(1 + \eta^2 L_{\max}^2) \|x^\star\|_2^2 + L_{\max} \eta^2 M  \\
=& f(x_t) - \left(\eta (1 - \theta - \eta L_{\max} M_f) - 2L_{\max}(1 + \eta^2 L_{\max}^2) \cdot \tfrac{\omega}{\mu^2}  \right) \cdot \| \nabla f(x_t)\|_2^2 
\\ 
&\: + 2\omega L_{\max}(1 + \eta^2 L_{\max}^2) \|x^\star\|_2^2 + L_{\max} \eta^2 M
\end{align*}
For coherence, assume that $\eta = \tfrac{1}{2L_{\max}}$.
Following the same procedure we obtain, we define $g\left(\tfrac{1}{2L_{\max}}\right) \equiv \alpha = \tfrac{1}{2L_{\max}} \left(\tfrac{1}{2} -\theta - \tfrac{M_f}{2}\right) - \tfrac{5L_{\max}}{2} \cdot \tfrac{\omega}{\mu^2}$. 
We observe that $\alpha > 0$ when $\omega < \tfrac{\mu^2}{5L_{\max}^2\left(\tfrac{1}{2} -\theta - \tfrac{M_f}{2}\right)}$.
Then, with similar reasoning with Theorem \ref{A:thm:1}, we telescope the inequality to obtain:
\begin{align*}
&\min_{t \in \{0, \dots, T\}} \mathbb{E}_{\mathcal{M}, i_t}\left[\|\nabla f(x_t)\|_2^2\right] \\
\leq & \tfrac{f(x_0) - f(x^\star)}{\alpha (T+1)} + \tfrac{1}{\alpha} \cdot \left(\tfrac{5L_{\max}\omega}{2}\cdot \|x^\star\|_2^2 + \tfrac{M}{4L_{\max}} \right)
\end{align*}
\end{proof}

\begin{acks}
This work is supported by NSF FET:Small no. 1907936, NSF MLWiNS CNS no. 2003137 (in collaboration with Intel), NSF CMMI no. 2037545, NSF CAREER award no. 2145629, and Rice InterDisciplinary Excellence Award (IDEA).
\end{acks} 

\bibliographystyle{ACM-Reference-Format}
\bibliography{sample}

\end{document}